%% file: main.tex
\documentclass[10pt,twocolumn,letterpaper]{article}

\usepackage{cvpr}              


\input{sec/custom_commands}

\definecolor{cvprblue}{rgb}{0.21,0.49,0.74}
\usepackage[pagebackref,breaklinks,colorlinks,citecolor=cvprblue]{hyperref}

\def\paperID{00000} 
\def\confName{CVPR}
\def\confYear{2025}

\title{Fast globally optimal Truncated Least Squares point cloud registration with fixed rotation axis}

\author{Ivo Ivanov, Carsten Markgraf\thanks{Corresponding author, E-Mail: ivo.ivanov@tha.de}\\
Technical University of Applied Sciences Augsburg, Germany\\
An der Hochschule 1, 86161 Augburg\\
}

\begin{document}
\maketitle
\input{sec/0_abstract}    
\input{sec/1_my_intro}

\input{sec/2_related_work}

\input{sec/3_tls_opt}

\input{sec/4_evaluation}

\input{sec/5_conclusion}

{
    \small
    \bibliographystyle{ieeenat_fullname}
    \bibliography{main}
}

\input{sec/X_suppl}

\end{document}

%% file: sec/custom_commands.tex
\usepackage{algorithmicx}
\usepackage{algorithm}
\usepackage[noend]{algpseudocode}
\usepackage{amsfonts,amsthm}
\usepackage[T1]{fontenc} 
\usepackage{adjustbox}
\usepackage{bm}
\usepackage{tikz}
\usepackage{multirow}
\usepackage{tkz-euclide}
\usetikzlibrary{matrix, positioning}  

\newcommand{\code}[1]{\texttt{#1}}

\newcommand{\Rone}{\mathbb{R}}
\newcommand{\Rtwo}{\mathbb{R}^2}
\newcommand{\Rthree}{\mathbb{R}^3}

\newcommand{\RthreeByN}{\mathbb{R}^{3 \times N}}

\newcommand{\SOtwo}{SO(2)}

\newcommand{\SOthree}{SO(3)}
\newcommand{\SEthree}{SE(\textrm{3})}

\theoremstyle{definition}
\newtheorem{definition}{Definition}
\newtheorem{theorem}{Theorem}
\newtheorem{lemma}[theorem]{Lemma}
\theoremstyle{remark}

\DeclareMathOperator*{\argmin}{argmin} 

\newcommand{\M}[1]{{\bm #1}} 
\newcommand{\MA}{\M{A}}
\newcommand{\MB}{\M{B}}
\newcommand{\MC}{\M{C}}

\newcommand{\MK}{\M{K}}

\newcommand{\MP}{\M{P}}
\newcommand{\MQ}{\M{Q}}

\newcommand{\MR}{\M{R}}

\newcommand{\MI}{\M{I}}

\newcommand{\MT}{\M{T}}

\newcommand{\vb}{\boldsymbol{b}}
\newcommand{\vc}{\boldsymbol{c}}
\newcommand{\vd}{\boldsymbol{d}}

\newcommand{\vg}{\boldsymbol{g}}

\newcommand{\vn}{\boldsymbol{n}}

\newcommand{\vp}{\boldsymbol{p}}
\newcommand{\vq}{\boldsymbol{q}}

\newcommand{\vu}{\boldsymbol{u}}
\newcommand{\vv}{\boldsymbol{v}}
\newcommand{\vt}{\boldsymbol{t}}
\newcommand{\vxx}{\boldsymbol{x}} 
\newcommand{\vy}{\boldsymbol{y}}

\newcommand{\vzz}{\boldsymbol{z}} 

\newcommand{\normsq}[1]{\left\|#1\right\|^2}
\newcommand{\norm}[1]{\left\| #1 \right\|}

\newcommand{\trace}{\textrm{tr}}
\newcommand{\transposed}{\mathsf{T}}

\newcommand{\Onlogn}{\mathcal{O}(N \log N)}

\newcommand{\MyAlgo}{OptiPose}

\newcommand{\RangeOneToN}{\ensuremath{ \{1, \cdots, N\} }}
\newcommand{\epssq}{\ensuremath{\epsilon^2}}


%% file: sec/0_abstract.tex
\begin{abstract}
Recent results showed that point cloud registration with given correspondences can be made robust to outlier rates of up to 95\% using the truncated least squares (TLS) formulation.
However, solving this combinatorial optimization problem to global optimality is challenging. Provably globally optimal approaches using semidefinite programming (SDP) relaxations take hundreds of seconds for 100 points. In this paper, we propose a novel linear time convex relaxation as well as a contractor method to speed up Branch and Bound (BnB). 
Our solver can register two 3D point clouds with 100 points to provable global optimality in less than half a second when the axis of rotation is provided. Although it currently cannot solve the full 6DoF problem, it is two orders of magnitude faster than the state-of-the-art SDP solver STRIDE when solving the rotation-only TLS problem.
In addition to providing a formal proof for global optimality, we present empirical evidence of global optimality using adversarial instances with local minimas close to the global minimum.

\end{abstract}

%% file: sec/1_my_intro.tex
\section{Introduction}
\label{sec:intro}

Outlier-robust point cloud registration is a fundamental and long-studied problem. It is the core of many applications such as scene reconstruction or localization in autonomous driving \cite{6094720, Yang20tro-teaser}. While available methods solve the problem satisfactorily, they offer no guarantees --  on hard instances with many outliers (e.g. 90\%) or symmetries in the scene they can fail  \cite{9785843}. In case of such a failure, an arbitrarily incorrect solution may be output.
For safety-critical systems such as autonomous vehicles, an attractive approach are globally optimal, also called \textit{certifiable}, algorithms. They are however only useful if they can be made fast enough to run in real-time.


Using keypoint detection and feature matching gives initial point correspondences, making the point cloud registration problem somewhat more tractable. Feature matching results however in many outliers, commonly up to $95\%$.
Robust estimators that can handle such extreme amounts of outliers have been previously proposed, such as the Maximum Consensus (MC) \cite{10161215, Chin2017TheMC} or the Truncated Least Squares (TLS) formulation \cite{yin2023outram}. However, they lead to combinatorial optimization problems that are very difficult to solve to globally optimality \cite[p.299]{Boyd_Vandenberghe_2004}. Branch-and Bound solvers can solve these problems but are usually too slow.

Recently, one promising direction has been to use semidefinite relaxations. It has been shown that the TLS problem can be written as a polynomial optimization problem and then use \textit{Lassere's hierarchy} to find semidefinite convex relaxations \cite{Yang20tro-teaser, Yang2020OneRT, 9785843}. In instances where this relaxation is tight, a global minimizer is obtained. 
Although such an SDP can be solved in polynomial time, the actual runtime is very slow, even with customized solvers such as STRIDE \cite{9785843}, taking hundreds of seconds for 100 points.

On the other hand, heuristic methods for approximately minimizing the TLS objective such as \textit{Graduated Non-Convexity} (GNC) \cite{8957085} or clique finding in a compatibility graph \cite{Yang20tro-teaser, 9562007, 10161215, zhang20233d} have been used with great success. Nevertheless, using such heuristic methods leads to the same problem of having no guarantee about global optimality, despite the overall great robustness to high amounts of outliers.

On the other hand, heuristic methods for solving the TLS objective such as \textit{Graduated Non-Convexity} (GNC) \cite{8957085} or clique finding in a compatibility graph \cite{Yang20tro-teaser, 9562007, 10161215, zhang20233d} have been used with great success. Nevertheless, using such heuristic methods leads to the same problem of having no guarantee about global optimality, despite the overall great robustness to high amounts of outliers.

From a statistical point of view, the TLS and MC problems can be thought of as a robust regression, selecting the best subset of data points. Datapoints are categorized into inliers and outliers based on a fixed threshold on the distance to the model.
In the case of point cloud registration, the model is a rigid transform of $\in \SEthree$ and each data point is a pair of points.

In the \textit{Maximum consensus} (MC) formulation, the objective is to maximize the number of inliers. In the \textit{Truncated Least Squares} (TLS) formulation, the residuals are considered additionally and the objective is to minimize a sum of truncated squared residuals.
Both estimators have been shown to be robust to high amounts of outliers, with TLS yielding slightly better results \cite{Yang20tro-teaser}.

In this paper, we focus on the TLS objective that has the advantage of being Lipschitz-continuous. It can therefore be minimized using Branch-and-Bound. We propose a Branch and Bound solver called \MyAlgo{} that uses a novel linear time convex relaxation based on interval analysis. The relaxation is simple and very inexpensive to compute and minimize. It leads to least-squares problems that are easily minimized using custom active-set solvers derived from existing least-squares solvers. Additionally, we propose an interval contractor method to significantly reduce the search space. Overall, our solver is able to solve the TLS point cloud registration problem in less than half a second, provided the axis of rotation.
Finally, we formally prove global optimality of our solver.


In summary, our main contributions are: 
\begin{itemize}
	\item A convex relaxation for the TLS objective function based interval analysis that can be computed in linear time
	\item Efficient custom active-set solvers for minimizing this convex relaxation
	\item A contractor method that significantly reduces the search space of Branch-and Bound
	\item A simple way to synthetically create adversarial instances with local minimas very close to the global minimum for validating global optimality empirically
\end{itemize}

%% file: sec/2_related_work.tex
\section{Related Work}
\label{sec:related-work}

Globally optimal, also called \textit{certifiable}, algorithms have been recently gaining interest in computer vision \cite{10023976, Convex-Relaxations-for-Pose-Graph-Optimization-With-Outliers} and robotics \cite{Planar-globally-optimal-pgo, 7759681, 10.1007/978-3-030-58539-6_18}. Specifically, both robust and outlier-free \cite{8100078, GARCIASALGUERO2023103862} point cloud registration have been studied besides other problems in computer vision \cite{10378410}, but remained impractical for real-time applications due to their long runtime. In the following, we briefly review algorithms for solving the robust point cloud registration problem with known correspondences, using either Truncated Least Squares (TLS) or the closely related Maximum Consensus (MC) formulation. 

\subsection{Heuristic algorithms}
Algorithms based on randomly sampling a subset of correspondences and checking whether they are all inliers, i.e. RANSAC were used for a long time. Torr et al. \cite{TORR2000138} proposed with MLESAC not only to maximize the number of inliers, but also to consider the residuals, an idea similar to the Truncated Least Squares objective. Recently, variants of RANSAC that consider geometric compatibility have been proposed \cite{9052691, 9552513}, improving outlier robustness up to 95\%.
Yang et al. \cite{Yang2020OneRT} presented TEASER++, a heuristic algorithm based on the TLS objective that provides many algorithmic insights and popularized the idea of finding cliques in compatibility graphs. They first construct a graph of compatible measurements and then prune outliers by finding the maximum clique. They then apply a dimensionality decomposition idea where they first solve for rotation using \textit{Graduated Non-Convexity} (GNC) and then separately over different dimensions of the translation using \textit{adaptive voting}. Separately from this heuristic algorithm, they propose a second globally optimal algorithm that uses semidefinite relaxations. However, it can only prove global optimality (i.e. \textit{certify}) for one subproblem at a time, i.e. rotation or translation, not the full 3D-transform. Finally, for global optimality they assume that an instance is not "too hard", i.e. they make assumptions unverifiable a priori  about an instance [Assump. 2, 3 in Thm.17]\cite{Yang20tro-teaser}. 
Other optimization methods for minimizing the TLS objective have been proposed, but they are either only heuristic, such as \cite{Barratt2020}, or not practical for higher dimensions \cite{doi:10.1080/10618600.2017.1390471}.

Methods using compatibility graphs are popular for point cloud registration \cite{Yang20tro-teaser, SC2-PCR-Chen-2022-CVPR, 10161215, lim2024kissmatcherfastrobustpoint} and outlier removal \cite{7410607, 10091912, zhang20233d}. They either list all maximal cliques, or search for the largest maximal clique (the \textit{maximum} clique). Some approaches such as ROBIN \cite{9562007} or CLIPPER \cite{10432947} approximate the maximum clique.
Recently, the concept of dimensionality decomposition has gained popularity. The idea is to decompose the 6DoF problem into rotation and translation \cite{Yang2020OneRT, 9485090}, or even more subproblems \cite{9878458, 10656079}, which are then solved sequentially.
However, solving these subproblems sequentially (even to global optimality) has not been proven to be equivalent to solving the original optimization problem -- as recently hinted, it very likely breaks global optimality \cite{9878458}.

\subsection{Exact algorithms}

Exact methods that guarantee to find the global optimum are mostly based on Branch-and-Bound (BnB) or semidefinite relaxations.  Branch-and-Bound algorithms have been proposed for solving the maximum consensus problem over rotations \cite{4408896, 10.1007/978-3-642-37444-9_42}. They usually have been slow however, even when searching over a 2D rotation \cite{9447984} using the nested BnB approach.

Another approach is to use tight convex relaxations based on Semidefinite Programs (SDPs). In their follow-up work to TEASER++, Yang et al. \cite{9785843} propose a sparse SDP relaxation for (among other problems) the full 6-DoF TLS point cloud registration problem and also a fast rank-1 SDP solver named \textit{STRIDE}. They can solve synthetic instances with $N=100$ points to global optimality with up to 80\% outlier rate, but this requires 200 seconds. 
SDPs have been used for similar problems such as multiple point cloud registration \cite{9157383}, but again with large runtimes.
One promising approach to solve large-scale SDP relaxations efficiently is the Burer-Monteiro approach \cite{papalia2024overviewburermonteiromethodcertifiable}. An extension of this method using Riemannian optimization named \textit{Riemannian staircase} lead to efficient and globally optimal algorithms for pose-graph optimization \cite{Rosen2017SESync, 10.1007/978-3-030-58539-6_18}. These problems are however outlier-free -- the Burer-Monteiro approach has not yet been successfully applied to problems with outliers \cite{papalia2024overviewburermonteiromethodcertifiable}.

%% file: sec/3_tls_opt.tex
\section{Problem formulation}
The problem is to compute a 3D-transform $(\MR^*, \vt^*) \in \SEthree$ that optimally aligns two point sequences $\mathcal{P} = \left(\vp_1, ..., \vp_N\right), \mathcal{Q} = \left(\vq_1, ..., \vq_N\right) \in \RthreeByN$ of noisy measurements. We know the correspondences between points ($\vp_i$ and $\vq_i$), but a large fraction (up to 95\%) can be outliers, as it is typical for feature-based matching methods.

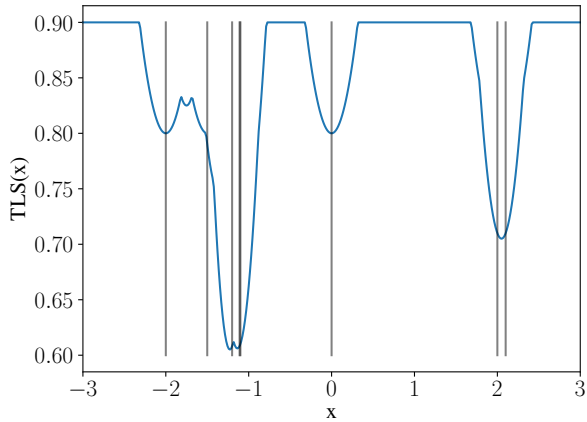
\begin{figure}[!ht]
	\centering
	\begin{adjustbox}{width=1.\linewidth}
		\input{figures/tls_cost_many_term.pgf}
	\end{adjustbox}
	\caption{The non-convex Truncated least-squares (TLS) function has multiple local minima, making global minimization difficult. Shown is the function  $\sum_{i=1}^{N}\min((x - y_i)^2, \epsilon^2)$ for $\epsilon^2=0.1$, and nine $y_i$'s (gray lines).}
	\label{fig:tlscostmulterm}
\end{figure}

To make the least squares estimation robust against those outliers, we use the Truncated Least Squares (TLS) formulation. In this formulation, residuals are truncated above a fixed threshold $\epsilon^2$. The point cloud registration problem in the TLS formulation is therefore:

\begin{equation}
	\label{eq:pcr-tls}
	\begin{aligned}
		(\MR^*, \vt^*) =  \argmin_{(\MR, \vt) \in \mathcal{R} \times \Rthree }  \sum_{i=1}^{N} \min \left(\normsq{\MR \vp_i  - \vq_i  + \vt}, \epsilon^2 \right)
	\end{aligned}
\end{equation}

In this paper, we optimize over 3D-rotations with fixed rotation axis $\vn^* \in \mathcal{S}^2$:
\begin{equation}
	\label{eq:fixed-axis-3d-transform}
	\begin{aligned}
		\mathcal{R} = \{ \exp( \theta \, [\vn^*]_{\times} ): \theta \in [-\pi, \pi] \}
	\end{aligned}
\end{equation}
where $\exp( \theta \, [\vn^*]_{\times} )$ maps the axis $\vn^*$ and angle $\theta$ to the corresponding 3D rotation matrix.

Intuitively, the parameter $\epsilon$ is the known measurement accuracy of each point. It can be estimated by collecting samples, i.e. in a data-driven manner \cite[p.8]{Chin2017TheMC}.

\begin{figure}[!ht]
	\centering
	\begin{adjustbox}{width=1.\linewidth}
		\includegraphics{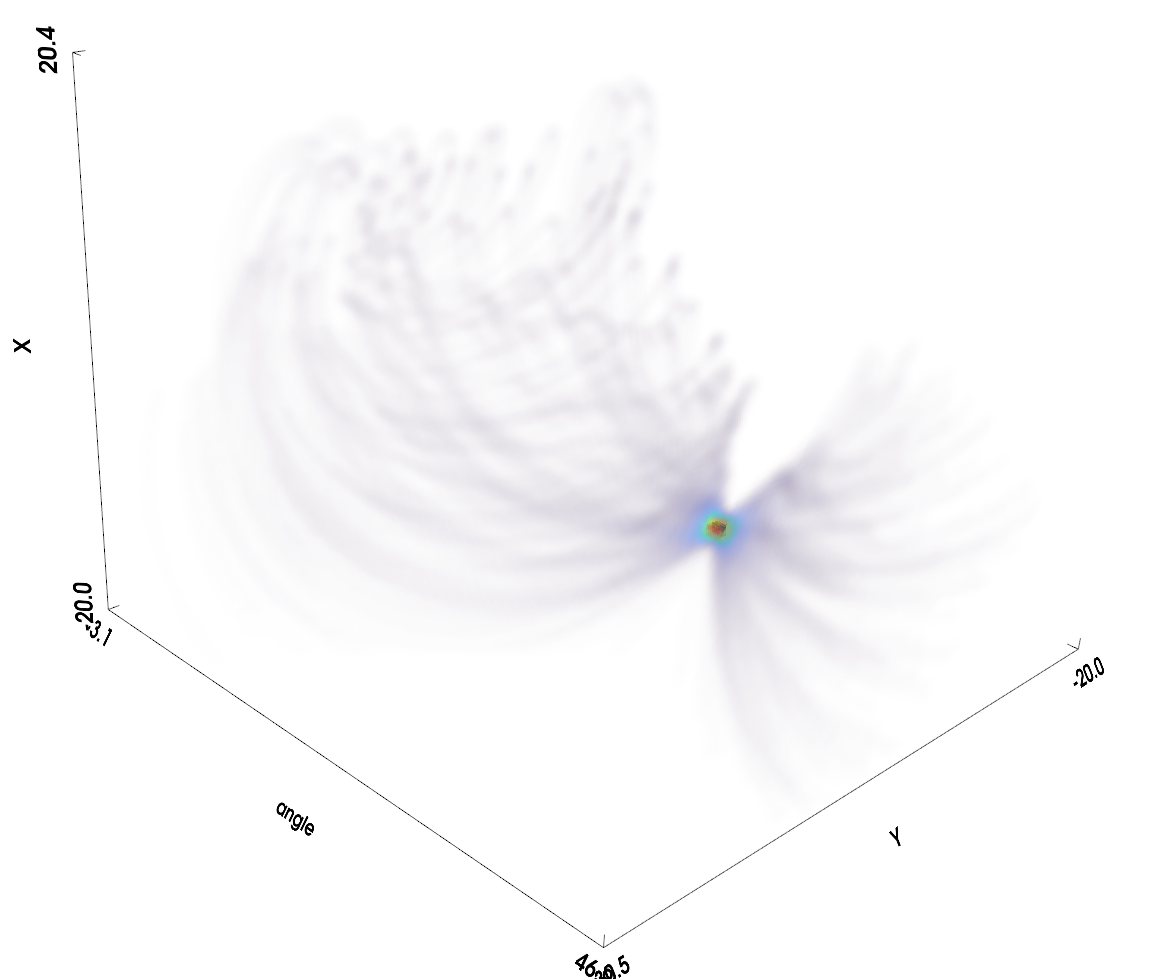}
	\end{adjustbox}
	\caption{The Truncated Least Squares objective function over $SE(2)$ (2D point cloud registration problem)}
	\label{fig:tls-objective}
\end{figure}
An example for the TLS objective is shown in Fig. \ref{fig:tlscostmulterm}.

Among the many ways to make least squares regression robust to outliers \cite[Ch. 3]{elements-of-stats-learning-book}, the TLS formulation has recently gained popularity \cite{Yang20tro-teaser, 9785843, NIPS2017_9f53d83e, doi:10.1080/10618600.2017.1390471, NIPS2010_01882513} since it has been shown to be very robust even to high outlier rates of 90\% \cite{Yang20tro-teaser}.
Solving the TLS problem however is challenging since it is a non-convex combinatorial problem \cite{5459398}.

\section{\MyAlgo{}: A fast BnB solver for the TLS registration}
In this section, we propose a fast Branch-and-Bound solver \MyAlgo{} for solving the TLS registration problem (\ref{eq:pcr-tls}).
It is a standard BnB algorithm with the main novel components being a linear time convex  relaxation as well as a contractor for reducing the search space.
\subsection{WLS Relaxation: A simple linear time relaxation}
In this section we propose a novel convex relaxation for the Truncated Least Squares problem that can be computed in linear time. Our relaxation can also be minimized easily by solving a weighted least squares problem, requiring only linear time.
We use this convex relaxation to compute a lower bound on the objective.

\begin{figure}[!ht]
	\centering
	\caption{The idea of the WLS convex relaxation: We compute the maximum $r_{max}$ of each convex term (orange) over the given interval (blue) and then re-weight it such that it remains below $\epsilon^2$ over this interval and is thus not truncated over the entire interval (green).}
	\label{fig:wlsrelaxideawithrelax}
\end{figure}

The idea behind the WLS relaxation is that the TLS objective is only non-convex due to the truncation, the residuals themselves are convex. Since the sum of convex functions is convex, we can relax each residual individually to make it convex.
This relaxation is computed for given subspaces of the search space, that is the current Branch and Bound node.
These nodes are generally balls parameterized by a center and radius. The relaxation is only valid (i.e., it underestimates the TLS objective) over these balls. 

To relax each residual, we first perform an interval analysis on it residual over the given ball. For example, consider the residual $r_i = \normsq{ \MR \vp_i - \vq_i }$. We first find the minimum and maximum of this residual over a ball of rotations, we denote them as $r_{min}^i$ and $r_{max}^i$ respectively. Then, we reweight these residuals so  that their maximum becomes $\epsilon^2$ and thus they are not truncated over the entire ball (Fig. \ref{fig:wlsrelaxideawithrelax}).
This leads to the following theorem:

\begin{theorem}
	\label{thm:wls-relax}
	The Truncated Least Squares objective is defined generally as
	
	\begin{equation}
		\label{eq:tls-t-opt-branch2}
		\begin{aligned}
			\text{TLS}(\vxx) = \sum_{i=1}^{N} \min \left(r(\vxx, \vxx_i)^2, \epsilon^2 \right)\\
		\end{aligned}
	\end{equation}
	
	where $\vxx_i$ are data points and $r(\vxx, \vxx_i)^2 = r_i$ are residual terms convex in $\vxx$. Given are the minimum and maximum value of each residual term $r_i$ over every $\vxx \in \mathcal{X}$, denoted as $r_{min}^i$ and $r_{max}^i$ respectively. Then, the following is a convex relaxation of (\ref{eq:tls-t-opt-branch2}) over $\mathcal{X}$:
	\begin{equation}
		\label{eq:tls-wls-relaxation}
		\begin{aligned}	
			\sum_{i=1}^{N} (1 - o_i) \left(w_i r_i  + (1 - w_i) r_{min}^i \right) + \sum_{i=1}^{N} o_i\epsilon^2\\
		\end{aligned}
	\end{equation}
	
	with the weights $w_i$ defined as:
	\begin{equation}
		\begin{aligned}	
			w_i  = 
			\begin{cases}
				\frac{\epsilon^2 - r_{min}^i}{r_{max}^i - r_{min}^i} &  r_{max}^i > \epsilon^2\\
				1 & \, \text{otherwise}
			\end{cases}
		\end{aligned}
	\end{equation}
	
	and the binary variables $o_i$ deciding whether the $i$-th residual is an outlier or not:
	\begin{equation}
		\begin{aligned}	
			o_i &= \begin{cases}
				1 &  r_{min}^i > \epsilon^2\\
				0 & \, \text{otherwise}
			\end{cases}
		\end{aligned}
	\end{equation}
	
\end{theorem}	

A proof is given in the Appendix \ref{proof:wls-relax}.

The main computational step is computing the weights.
Additional case handling is required to ensure that the relaxation becomes tight as the ball radius approaches zero. This is crucial to ensure that BnB converges in a finite number of iterations. First, all residuals with  $r_{min} \ge \epsilon^2$ are outliers over the entire interval and therefore constant. Similarly, residuals with $r_{max} \leq \epsilon^2$ are inliers over the entire interval and do not require reweighting.

\subsection{Interval analysis of different residuals}

In the following we show how the interval $[r_{min}, r_{max}]$ can easily be computed analytically for different models (i.e. translation and rotation), given that the subspaces are balls.

\begin{theorem}
	\label{thm:interval-analysis-tranlation}
	Given a ball $\MB^t := \{ \vt \in \Rthree : \norm{ \vc - \vt } \leq n_r \}$, the residual $\normsq{ \vp - \vq + \vt } \in [r_{min}, r_{max}] $ for all $\vt \in \MB^t$, where 
	\begin{equation}
		\label{eq:t-interval}
		\begin{aligned}
			r_{min} &= \max(0, \norm{ \vd } - n_r)^2\\
			r_{max} &= (\norm{ \vd } + n_r)^2\\
			\vd &= \vp - \vq + \vc
		\end{aligned}
	\end{equation}
\end{theorem}
See App. \ref{proof:tls-relax-minmax-r} for a proof.

Similarly, for the rotation residual $\normsq{ \MR \vp - \vq }$ over a ball with center $\MR_c$ and (geodesic) radius $n_r$ we have:

\begin{theorem}
	\label{thm:interval-analysis-rotation}
	Given a ball $\MB^r := \{ \MR \in \SOthree : d_{\angle}(\MR, \MR_c) \leq n_r\}$, the residual $\normsq{ \MR \vp - \vq } \in [r_{min}, r_{max}]$ for all 
	$\MR \in \SOthree \in \MB^r$, where
	\begin{equation}
		\label{eq:r-interval}
		\begin{aligned}
			r_{min} &= \normsq{\vp} + \normsq{\vq} - 2 \norm{\vp} \, \norm{\vq} \cos (\max(\theta - n_r, 0))\\
			r_{max} &= \normsq{\vp} + \normsq{\vq} - 2 \norm{\vp} \, \norm{\vq} \cos (\min(\theta + n_r, \pi))
		\end{aligned}
	\end{equation}
	and
	\begin{equation}
		\begin{aligned}
			\theta &= \arccos \left( \frac{ \vp^\transposed \MR_c^\transposed \vq}
			{\norm{\vp} \, \norm{\vq} } \right)
		\end{aligned}
	\end{equation}
	
\end{theorem} 
See App. \ref{proof:wls-interval-analysis-rotation} for a proof.

Computing these intervals requires constant time and is in practice very fast since the trigonometric functions can be eliminated.

\textbf{3D Transform}. To compute the interval for the residual $\normsq{ \MR \vp - \vq  + \vt}$ over a rotation and translation ball, note that the translation interval Eq. \ref{eq:t-interval} only depends on the norm $\norm{\vd}$. Therefore, we can simply substitute $ \MR \vp - \vq$ for $\vd$ to compute the interval accordingly. 

\subsection{Minimizing the WLS convex relaxation}
After computing the weighted least squares (WLS) relaxation, we need to minimize it to obtain a lower bound on the TLS objective.
Since the relaxation is a weighted least squares function, it can be minimized using existing least squares solvers. However, the solution must be within the given Branch and Bound node. Therefore, we must introduce a ball constraint into the minimization problem.
This constrained problem can be solved efficiently using custom active-set solvers, which we present in the following.

\textbf{Minimizing the WLS-relaxation over rotations}. The problem of minimizing the WLS relaxation over a geodesic ball of rotations $\{\MR \in \SOthree:  d_{\angle}(\MR_c, \MR) \leq n_r\}$ is:

\begin{equation}
	\label{eq:tls-rwls}
	\begin{aligned}
		\argmin_{\MR \in \SOthree}  \quad &\sum_{i=1}^{N} w_i \normsq{ \MR\vp_i - \vq_i } \\
		\text{subject to} \quad &d_{\angle}(\MR_c, \MR) \leq n_r
	\end{aligned}
\end{equation}

where $\MR_c$ is the ball center and $n_r$ is the radius, as in the translation case. $d_{\angle}(\cdot)$ is the geodesic (also called angular) distance.
We can write this problem equivalently as:
\begin{equation}
	\label{eq:tls-rwls-centered}
	\begin{aligned}
		\argmin_{\Delta \MR \in \SOthree}  \quad &\sum_{i=1}^{N} w_i \normsq{ \Delta\MR \vp_i - \MR_c^\transposed \vq_i } \\
		\text{subject to} \quad &d_{\angle}(\MI, \Delta\MR) \leq n_r
	\end{aligned}
\end{equation}

where $\MR = \MR_c\Delta\MR$, $\Delta\MR$ is the rotation relative to the node's center and $\MI$ is the identity rotation.
The unconstrained version of this problem is known as the \textit{Wahba-problem} for which a simple SVD-algorithm exists \cite{Kabsch-1978-Point-set-alignment, Lawrence2019APA, Least-squares-estimation-point-sets-Umeyama-1991}. 
We can solve the constrained problem efficiently as well using a custom active-set solver.

The idea of the active set method is based on the observation that every inequality constraint can either be ignored or is satisfied as an equality \cite[p.424]{Numerical-Optimization-Nocedal-Wright}. We therefore first solve the problem using the SVD algorithm and thereby ignore the constraint. Then, we simply check whether the constraint is already satisfied. If yes, we have found the solution. If not, the constraint is satisfied as an equality and we need to find a 3D rotation that has a rotation angle equal to $n_r$. According to the Euler theorem, every 3D rotation can be described as a rotation around a single axis and a single angle. Since we know that the rotation angle must be $n_r$, we only need to find the rotation axis.

These two steps can be performed efficiently in practice since the second step is based on the same $3 \times 3$ cross-correlation matrix as the first step (the SVD algorithm). 
In the following, we elaborate on how we solve the second step, i.e. the equality-constrained case.

\begin{theorem}
	\label{thm:constrained-wahba-davenport}
	The optimal axis of rotation $\vn^*$ of the following least-squares rotation estimation problem 
	\begin{equation}
		\label{eq:tls-rwls-a}
		\begin{aligned}
			\argmin_{\MR \in \SOthree}  \quad &\sum_{i=1}^{N} w_i \normsq{ \MR \vp_i -  \vq_i } \\
			\text{subject to} \quad &d_{\angle}(\MI, \MR) = n_r
		\end{aligned}
	\end{equation}
	where the angle of rotation is constrained to be equal to $n_r$ is obtained equivalently by solving the following optimization problem:
	\begin{equation}
		\label{eq:rot-est-constrained-qcqp}
		\begin{aligned}
			\vn^* = \argmin_{\vn \in \Rthree} \quad &-\left(\vn^\transposed \MA \vn + 2 \vg^\transposed \vn \right)\\
			\text{subject to} \quad &|| \vn|| = 1\\
		\end{aligned}
	\end{equation}
\end{theorem}
See App. \ref{proof:wls-relax-eq-constrained-rotation} for a proof.
$\MA$ and $\vg$ are data matrices that depend on the points $\vp_i$, $\vq_i$ and the angle constraint $n_r$. A definition is given in proof.

The non-convex problem (\ref{eq:rot-est-constrained-qcqp}) over the sphere can be solved to global optimality using an algorithm based on eigen-decomposition, proposed in \cite{10.1007/978-3-642-75536-1_57}. First, it computes the eigen-decomposition  $\MA = \mathbf{Q}\mathbf{\Lambda}\mathbf{Q}^\transposed$ and then finds the rightmost root of a certain rational function. This is therefore a constant time algorithm in the number of points $N$. For details, we refer to \cite{10.1007/978-3-642-75536-1_57}.

\textbf{Minimizing the WLS-relaxation over 3D-transforms}. 
To minimize the WLS relaxation, we need to do least-squares point cloud registration with two ball constraints (on the rotation and translation):

\begin{equation}
	\label{eq:wls-pose}
	\begin{aligned}
		(\MR^*, \vt^*) =  \argmin_{(\MR, \vt) \in \mathcal{R} \times  \Rthree}  &\sum_{i=1}^{N} w_i \normsq{\MR \vp_i  - \vq_i  + \vt} \\
		\text{subject to} \quad &d_{\angle}(\MR_c, \MR) \leq n_r\\
			\quad  & \norm{\vc - \vt}  \leq n_t\\
	\end{aligned}
\end{equation}
We propose a modified Kabsch-Umeyama active-set algorithm to solve this problem. It involves solving the rotation problem (\ref{eq:tls-rwls}) as the main step.
See App. \ref{proof:minimizing-wls-relaxation} for details.

\input{sec/contractor}

\subsection{Proof of global optimality} 

In this section we prove the global optimality of our Branch and Bound solver \MyAlgo{} for solving the TLS problem (\ref{eq:pcr-tls}). Branch-and-bound (BnB) is a general framework for designing globally optimal algorithms that utilizes problem-specific algorithmic components, such as bounds computation and many other heuristics \cite{Misener2013}.
The proposed algorithm \MyAlgo{} is a generic best-first Branch-and-Bound algorithm that uses only the aforementioned components: our WLS convex relaxation and the proposed contractor.
Therefore, we only prove the correctness of these components and refer the reader to \cite{bnb-script-boyd} for a general reference on the global optimality of BnB.

Overall, we prove the following properties which are sufficient to prove global optimality:
\begin{enumerate}
	\item The lower and upper bounds are always valid 
	\item The contractor has the completeness property (i.e. does not lose the optimum)
	\item The global minimizer is contained in the root node
	\item For every given node, the branching step creates a partition of this node
\end{enumerate}
The main part of the proof involves proving the first two properties. The proof details can be found in the appendix.
We prove that (1) our proposed WLS relaxation is indeed a relaxation (see App. \ref{proof:wls-relax}), and that (2) the proposed solver always minimizes this relaxation to global optimality (see App. \ref{proof:minimizing-wls-relaxation}). 
In App. \ref{proof:contractor}, we prove the proposed search space contractor has the completeness property.

To obtain the upper bound, we use two heuristics: (1) \textit{Graduated Non-Convexity} (GNC) \cite{Yang20tro-teaser, 8957085}, and (2) evaluating the TLS objective at the minimizer of the convex relaxation. 
Since a valid upper bound can be obtained by evaluating the objective function at any point inside the current BnB node, the second method yields a valid upper bound. The first method, GNC, also yields a valid upper bound when applied to the root node.

Regarding the root node, for the rotation we start from the complete space $\mathcal{R}$, that trivially contains the global minimizer. For the translation root node, see App. \ref{proof:root-box}.

The branching step over translations generally uses cubes and binary split, yielding a valid partition.
Regarding branching over the space of 3D rotations, we use the rotation vector approach proposed in \cite{HartleyKahl2009GlobRotEst} where the authors prove it is a valid partition.

%% file: figures/tls_cost_many_term.pgf
\begingroup%
\makeatletter%
\begin{pgfpicture}%
\pgfpathrectangle{\pgfpointorigin}{\pgfqpoint{6.400000in}{4.800000in}}%
\pgfusepath{use as bounding box, clip}%
\begin{pgfscope}%
\pgfsetbuttcap%
\pgfsetmiterjoin%
\definecolor{currentfill}{rgb}{1.000000,1.000000,1.000000}%
\pgfsetfillcolor{currentfill}%
\pgfsetlinewidth{0.000000pt}%
\definecolor{currentstroke}{rgb}{1.000000,1.000000,1.000000}%
\pgfsetstrokecolor{currentstroke}%
\pgfsetdash{}{0pt}%
\pgfpathmoveto{\pgfqpoint{0.000000in}{0.000000in}}%
\pgfpathlineto{\pgfqpoint{6.400000in}{0.000000in}}%
\pgfpathlineto{\pgfqpoint{6.400000in}{4.800000in}}%
\pgfpathlineto{\pgfqpoint{0.000000in}{4.800000in}}%
\pgfpathlineto{\pgfqpoint{0.000000in}{0.000000in}}%
\pgfpathclose%
\pgfusepath{fill}%
\end{pgfscope}%
\begin{pgfscope}%
\pgfsetbuttcap%
\pgfsetmiterjoin%
\definecolor{currentfill}{rgb}{1.000000,1.000000,1.000000}%
\pgfsetfillcolor{currentfill}%
\pgfsetlinewidth{0.000000pt}%
\definecolor{currentstroke}{rgb}{0.000000,0.000000,0.000000}%
\pgfsetstrokecolor{currentstroke}%
\pgfsetstrokeopacity{0.000000}%
\pgfsetdash{}{0pt}%
\pgfpathmoveto{\pgfqpoint{1.027842in}{0.819475in}}%
\pgfpathlineto{\pgfqpoint{6.104966in}{0.819475in}}%
\pgfpathlineto{\pgfqpoint{6.104966in}{4.560000in}}%
\pgfpathlineto{\pgfqpoint{1.027842in}{4.560000in}}%
\pgfpathlineto{\pgfqpoint{1.027842in}{0.819475in}}%
\pgfpathclose%
\pgfusepath{fill}%
\end{pgfscope}%
\begin{pgfscope}%
\pgfsetbuttcap%
\pgfsetroundjoin%
\definecolor{currentfill}{rgb}{0.000000,0.000000,0.000000}%
\pgfsetfillcolor{currentfill}%
\pgfsetlinewidth{0.803000pt}%
\definecolor{currentstroke}{rgb}{0.000000,0.000000,0.000000}%
\pgfsetstrokecolor{currentstroke}%
\pgfsetdash{}{0pt}%
\pgfsys@defobject{currentmarker}{\pgfqpoint{0.000000in}{-0.048611in}}{\pgfqpoint{0.000000in}{0.000000in}}{%
\pgfpathmoveto{\pgfqpoint{0.000000in}{0.000000in}}%
\pgfpathlineto{\pgfqpoint{0.000000in}{-0.048611in}}%
\pgfusepath{stroke,fill}%
}%
\begin{pgfscope}%
\pgfsys@transformshift{1.027842in}{0.819475in}%
\pgfsys@useobject{currentmarker}{}%
\end{pgfscope}%
\end{pgfscope}%
\begin{pgfscope}%
\definecolor{textcolor}{rgb}{0.000000,0.000000,0.000000}%
\pgfsetstrokecolor{textcolor}%
\pgfsetfillcolor{textcolor}%
\pgftext[x=1.027842in,y=0.722253in,,top]{\color{textcolor}\rmfamily\fontsize{16.000000}{19.200000}\selectfont \(\displaystyle {\ensuremath{-}3}\)}%
\end{pgfscope}%
\begin{pgfscope}%
\pgfsetbuttcap%
\pgfsetroundjoin%
\definecolor{currentfill}{rgb}{0.000000,0.000000,0.000000}%
\pgfsetfillcolor{currentfill}%
\pgfsetlinewidth{0.803000pt}%
\definecolor{currentstroke}{rgb}{0.000000,0.000000,0.000000}%
\pgfsetstrokecolor{currentstroke}%
\pgfsetdash{}{0pt}%
\pgfsys@defobject{currentmarker}{\pgfqpoint{0.000000in}{-0.048611in}}{\pgfqpoint{0.000000in}{0.000000in}}{%
\pgfpathmoveto{\pgfqpoint{0.000000in}{0.000000in}}%
\pgfpathlineto{\pgfqpoint{0.000000in}{-0.048611in}}%
\pgfusepath{stroke,fill}%
}%
\begin{pgfscope}%
\pgfsys@transformshift{1.874030in}{0.819475in}%
\pgfsys@useobject{currentmarker}{}%
\end{pgfscope}%
\end{pgfscope}%
\begin{pgfscope}%
\definecolor{textcolor}{rgb}{0.000000,0.000000,0.000000}%
\pgfsetstrokecolor{textcolor}%
\pgfsetfillcolor{textcolor}%
\pgftext[x=1.874030in,y=0.722253in,,top]{\color{textcolor}\rmfamily\fontsize{16.000000}{19.200000}\selectfont \(\displaystyle {\ensuremath{-}2}\)}%
\end{pgfscope}%
\begin{pgfscope}%
\pgfsetbuttcap%
\pgfsetroundjoin%
\definecolor{currentfill}{rgb}{0.000000,0.000000,0.000000}%
\pgfsetfillcolor{currentfill}%
\pgfsetlinewidth{0.803000pt}%
\definecolor{currentstroke}{rgb}{0.000000,0.000000,0.000000}%
\pgfsetstrokecolor{currentstroke}%
\pgfsetdash{}{0pt}%
\pgfsys@defobject{currentmarker}{\pgfqpoint{0.000000in}{-0.048611in}}{\pgfqpoint{0.000000in}{0.000000in}}{%
\pgfpathmoveto{\pgfqpoint{0.000000in}{0.000000in}}%
\pgfpathlineto{\pgfqpoint{0.000000in}{-0.048611in}}%
\pgfusepath{stroke,fill}%
}%
\begin{pgfscope}%
\pgfsys@transformshift{2.720217in}{0.819475in}%
\pgfsys@useobject{currentmarker}{}%
\end{pgfscope}%
\end{pgfscope}%
\begin{pgfscope}%
\definecolor{textcolor}{rgb}{0.000000,0.000000,0.000000}%
\pgfsetstrokecolor{textcolor}%
\pgfsetfillcolor{textcolor}%
\pgftext[x=2.720217in,y=0.722253in,,top]{\color{textcolor}\rmfamily\fontsize{16.000000}{19.200000}\selectfont \(\displaystyle {\ensuremath{-}1}\)}%
\end{pgfscope}%
\begin{pgfscope}%
\pgfsetbuttcap%
\pgfsetroundjoin%
\definecolor{currentfill}{rgb}{0.000000,0.000000,0.000000}%
\pgfsetfillcolor{currentfill}%
\pgfsetlinewidth{0.803000pt}%
\definecolor{currentstroke}{rgb}{0.000000,0.000000,0.000000}%
\pgfsetstrokecolor{currentstroke}%
\pgfsetdash{}{0pt}%
\pgfsys@defobject{currentmarker}{\pgfqpoint{0.000000in}{-0.048611in}}{\pgfqpoint{0.000000in}{0.000000in}}{%
\pgfpathmoveto{\pgfqpoint{0.000000in}{0.000000in}}%
\pgfpathlineto{\pgfqpoint{0.000000in}{-0.048611in}}%
\pgfusepath{stroke,fill}%
}%
\begin{pgfscope}%
\pgfsys@transformshift{3.566404in}{0.819475in}%
\pgfsys@useobject{currentmarker}{}%
\end{pgfscope}%
\end{pgfscope}%
\begin{pgfscope}%
\definecolor{textcolor}{rgb}{0.000000,0.000000,0.000000}%
\pgfsetstrokecolor{textcolor}%
\pgfsetfillcolor{textcolor}%
\pgftext[x=3.566404in,y=0.722253in,,top]{\color{textcolor}\rmfamily\fontsize{16.000000}{19.200000}\selectfont \(\displaystyle {0}\)}%
\end{pgfscope}%
\begin{pgfscope}%
\pgfsetbuttcap%
\pgfsetroundjoin%
\definecolor{currentfill}{rgb}{0.000000,0.000000,0.000000}%
\pgfsetfillcolor{currentfill}%
\pgfsetlinewidth{0.803000pt}%
\definecolor{currentstroke}{rgb}{0.000000,0.000000,0.000000}%
\pgfsetstrokecolor{currentstroke}%
\pgfsetdash{}{0pt}%
\pgfsys@defobject{currentmarker}{\pgfqpoint{0.000000in}{-0.048611in}}{\pgfqpoint{0.000000in}{0.000000in}}{%
\pgfpathmoveto{\pgfqpoint{0.000000in}{0.000000in}}%
\pgfpathlineto{\pgfqpoint{0.000000in}{-0.048611in}}%
\pgfusepath{stroke,fill}%
}%
\begin{pgfscope}%
\pgfsys@transformshift{4.412591in}{0.819475in}%
\pgfsys@useobject{currentmarker}{}%
\end{pgfscope}%
\end{pgfscope}%
\begin{pgfscope}%
\definecolor{textcolor}{rgb}{0.000000,0.000000,0.000000}%
\pgfsetstrokecolor{textcolor}%
\pgfsetfillcolor{textcolor}%
\pgftext[x=4.412591in,y=0.722253in,,top]{\color{textcolor}\rmfamily\fontsize{16.000000}{19.200000}\selectfont \(\displaystyle {1}\)}%
\end{pgfscope}%
\begin{pgfscope}%
\pgfsetbuttcap%
\pgfsetroundjoin%
\definecolor{currentfill}{rgb}{0.000000,0.000000,0.000000}%
\pgfsetfillcolor{currentfill}%
\pgfsetlinewidth{0.803000pt}%
\definecolor{currentstroke}{rgb}{0.000000,0.000000,0.000000}%
\pgfsetstrokecolor{currentstroke}%
\pgfsetdash{}{0pt}%
\pgfsys@defobject{currentmarker}{\pgfqpoint{0.000000in}{-0.048611in}}{\pgfqpoint{0.000000in}{0.000000in}}{%
\pgfpathmoveto{\pgfqpoint{0.000000in}{0.000000in}}%
\pgfpathlineto{\pgfqpoint{0.000000in}{-0.048611in}}%
\pgfusepath{stroke,fill}%
}%
\begin{pgfscope}%
\pgfsys@transformshift{5.258779in}{0.819475in}%
\pgfsys@useobject{currentmarker}{}%
\end{pgfscope}%
\end{pgfscope}%
\begin{pgfscope}%
\definecolor{textcolor}{rgb}{0.000000,0.000000,0.000000}%
\pgfsetstrokecolor{textcolor}%
\pgfsetfillcolor{textcolor}%
\pgftext[x=5.258779in,y=0.722253in,,top]{\color{textcolor}\rmfamily\fontsize{16.000000}{19.200000}\selectfont \(\displaystyle {2}\)}%
\end{pgfscope}%
\begin{pgfscope}%
\pgfsetbuttcap%
\pgfsetroundjoin%
\definecolor{currentfill}{rgb}{0.000000,0.000000,0.000000}%
\pgfsetfillcolor{currentfill}%
\pgfsetlinewidth{0.803000pt}%
\definecolor{currentstroke}{rgb}{0.000000,0.000000,0.000000}%
\pgfsetstrokecolor{currentstroke}%
\pgfsetdash{}{0pt}%
\pgfsys@defobject{currentmarker}{\pgfqpoint{0.000000in}{-0.048611in}}{\pgfqpoint{0.000000in}{0.000000in}}{%
\pgfpathmoveto{\pgfqpoint{0.000000in}{0.000000in}}%
\pgfpathlineto{\pgfqpoint{0.000000in}{-0.048611in}}%
\pgfusepath{stroke,fill}%
}%
\begin{pgfscope}%
\pgfsys@transformshift{6.104966in}{0.819475in}%
\pgfsys@useobject{currentmarker}{}%
\end{pgfscope}%
\end{pgfscope}%
\begin{pgfscope}%
\definecolor{textcolor}{rgb}{0.000000,0.000000,0.000000}%
\pgfsetstrokecolor{textcolor}%
\pgfsetfillcolor{textcolor}%
\pgftext[x=6.104966in,y=0.722253in,,top]{\color{textcolor}\rmfamily\fontsize{16.000000}{19.200000}\selectfont \(\displaystyle {3}\)}%
\end{pgfscope}%
\begin{pgfscope}%
\definecolor{textcolor}{rgb}{0.000000,0.000000,0.000000}%
\pgfsetstrokecolor{textcolor}%
\pgfsetfillcolor{textcolor}%
\pgftext[x=3.566404in,y=0.453349in,,top]{\color{textcolor}\rmfamily\fontsize{16.000000}{19.200000}\selectfont x}%
\end{pgfscope}%
\begin{pgfscope}%
\pgfsetbuttcap%
\pgfsetroundjoin%
\definecolor{currentfill}{rgb}{0.000000,0.000000,0.000000}%
\pgfsetfillcolor{currentfill}%
\pgfsetlinewidth{0.803000pt}%
\definecolor{currentstroke}{rgb}{0.000000,0.000000,0.000000}%
\pgfsetstrokecolor{currentstroke}%
\pgfsetdash{}{0pt}%
\pgfsys@defobject{currentmarker}{\pgfqpoint{-0.048611in}{0.000000in}}{\pgfqpoint{-0.000000in}{0.000000in}}{%
\pgfpathmoveto{\pgfqpoint{-0.000000in}{0.000000in}}%
\pgfpathlineto{\pgfqpoint{-0.048611in}{0.000000in}}%
\pgfusepath{stroke,fill}%
}%
\begin{pgfscope}%
\pgfsys@transformshift{1.027842in}{0.989499in}%
\pgfsys@useobject{currentmarker}{}%
\end{pgfscope}%
\end{pgfscope}%
\begin{pgfscope}%
\definecolor{textcolor}{rgb}{0.000000,0.000000,0.000000}%
\pgfsetstrokecolor{textcolor}%
\pgfsetfillcolor{textcolor}%
\pgftext[x=0.535138in, y=0.906166in, left, base]{\color{textcolor}\rmfamily\fontsize{16.000000}{19.200000}\selectfont \(\displaystyle {0.60}\)}%
\end{pgfscope}%
\begin{pgfscope}%
\pgfsetbuttcap%
\pgfsetroundjoin%
\definecolor{currentfill}{rgb}{0.000000,0.000000,0.000000}%
\pgfsetfillcolor{currentfill}%
\pgfsetlinewidth{0.803000pt}%
\definecolor{currentstroke}{rgb}{0.000000,0.000000,0.000000}%
\pgfsetstrokecolor{currentstroke}%
\pgfsetdash{}{0pt}%
\pgfsys@defobject{currentmarker}{\pgfqpoint{-0.048611in}{0.000000in}}{\pgfqpoint{-0.000000in}{0.000000in}}{%
\pgfpathmoveto{\pgfqpoint{-0.000000in}{0.000000in}}%
\pgfpathlineto{\pgfqpoint{-0.048611in}{0.000000in}}%
\pgfusepath{stroke,fill}%
}%
\begin{pgfscope}%
\pgfsys@transformshift{1.027842in}{1.556245in}%
\pgfsys@useobject{currentmarker}{}%
\end{pgfscope}%
\end{pgfscope}%
\begin{pgfscope}%
\definecolor{textcolor}{rgb}{0.000000,0.000000,0.000000}%
\pgfsetstrokecolor{textcolor}%
\pgfsetfillcolor{textcolor}%
\pgftext[x=0.535138in, y=1.472912in, left, base]{\color{textcolor}\rmfamily\fontsize{16.000000}{19.200000}\selectfont \(\displaystyle {0.65}\)}%
\end{pgfscope}%
\begin{pgfscope}%
\pgfsetbuttcap%
\pgfsetroundjoin%
\definecolor{currentfill}{rgb}{0.000000,0.000000,0.000000}%
\pgfsetfillcolor{currentfill}%
\pgfsetlinewidth{0.803000pt}%
\definecolor{currentstroke}{rgb}{0.000000,0.000000,0.000000}%
\pgfsetstrokecolor{currentstroke}%
\pgfsetdash{}{0pt}%
\pgfsys@defobject{currentmarker}{\pgfqpoint{-0.048611in}{0.000000in}}{\pgfqpoint{-0.000000in}{0.000000in}}{%
\pgfpathmoveto{\pgfqpoint{-0.000000in}{0.000000in}}%
\pgfpathlineto{\pgfqpoint{-0.048611in}{0.000000in}}%
\pgfusepath{stroke,fill}%
}%
\begin{pgfscope}%
\pgfsys@transformshift{1.027842in}{2.122991in}%
\pgfsys@useobject{currentmarker}{}%
\end{pgfscope}%
\end{pgfscope}%
\begin{pgfscope}%
\definecolor{textcolor}{rgb}{0.000000,0.000000,0.000000}%
\pgfsetstrokecolor{textcolor}%
\pgfsetfillcolor{textcolor}%
\pgftext[x=0.535138in, y=2.039658in, left, base]{\color{textcolor}\rmfamily\fontsize{16.000000}{19.200000}\selectfont \(\displaystyle {0.70}\)}%
\end{pgfscope}%
\begin{pgfscope}%
\pgfsetbuttcap%
\pgfsetroundjoin%
\definecolor{currentfill}{rgb}{0.000000,0.000000,0.000000}%
\pgfsetfillcolor{currentfill}%
\pgfsetlinewidth{0.803000pt}%
\definecolor{currentstroke}{rgb}{0.000000,0.000000,0.000000}%
\pgfsetstrokecolor{currentstroke}%
\pgfsetdash{}{0pt}%
\pgfsys@defobject{currentmarker}{\pgfqpoint{-0.048611in}{0.000000in}}{\pgfqpoint{-0.000000in}{0.000000in}}{%
\pgfpathmoveto{\pgfqpoint{-0.000000in}{0.000000in}}%
\pgfpathlineto{\pgfqpoint{-0.048611in}{0.000000in}}%
\pgfusepath{stroke,fill}%
}%
\begin{pgfscope}%
\pgfsys@transformshift{1.027842in}{2.689737in}%
\pgfsys@useobject{currentmarker}{}%
\end{pgfscope}%
\end{pgfscope}%
\begin{pgfscope}%
\definecolor{textcolor}{rgb}{0.000000,0.000000,0.000000}%
\pgfsetstrokecolor{textcolor}%
\pgfsetfillcolor{textcolor}%
\pgftext[x=0.535138in, y=2.606404in, left, base]{\color{textcolor}\rmfamily\fontsize{16.000000}{19.200000}\selectfont \(\displaystyle {0.75}\)}%
\end{pgfscope}%
\begin{pgfscope}%
\pgfsetbuttcap%
\pgfsetroundjoin%
\definecolor{currentfill}{rgb}{0.000000,0.000000,0.000000}%
\pgfsetfillcolor{currentfill}%
\pgfsetlinewidth{0.803000pt}%
\definecolor{currentstroke}{rgb}{0.000000,0.000000,0.000000}%
\pgfsetstrokecolor{currentstroke}%
\pgfsetdash{}{0pt}%
\pgfsys@defobject{currentmarker}{\pgfqpoint{-0.048611in}{0.000000in}}{\pgfqpoint{-0.000000in}{0.000000in}}{%
\pgfpathmoveto{\pgfqpoint{-0.000000in}{0.000000in}}%
\pgfpathlineto{\pgfqpoint{-0.048611in}{0.000000in}}%
\pgfusepath{stroke,fill}%
}%
\begin{pgfscope}%
\pgfsys@transformshift{1.027842in}{3.256484in}%
\pgfsys@useobject{currentmarker}{}%
\end{pgfscope}%
\end{pgfscope}%
\begin{pgfscope}%
\definecolor{textcolor}{rgb}{0.000000,0.000000,0.000000}%
\pgfsetstrokecolor{textcolor}%
\pgfsetfillcolor{textcolor}%
\pgftext[x=0.535138in, y=3.173150in, left, base]{\color{textcolor}\rmfamily\fontsize{16.000000}{19.200000}\selectfont \(\displaystyle {0.80}\)}%
\end{pgfscope}%
\begin{pgfscope}%
\pgfsetbuttcap%
\pgfsetroundjoin%
\definecolor{currentfill}{rgb}{0.000000,0.000000,0.000000}%
\pgfsetfillcolor{currentfill}%
\pgfsetlinewidth{0.803000pt}%
\definecolor{currentstroke}{rgb}{0.000000,0.000000,0.000000}%
\pgfsetstrokecolor{currentstroke}%
\pgfsetdash{}{0pt}%
\pgfsys@defobject{currentmarker}{\pgfqpoint{-0.048611in}{0.000000in}}{\pgfqpoint{-0.000000in}{0.000000in}}{%
\pgfpathmoveto{\pgfqpoint{-0.000000in}{0.000000in}}%
\pgfpathlineto{\pgfqpoint{-0.048611in}{0.000000in}}%
\pgfusepath{stroke,fill}%
}%
\begin{pgfscope}%
\pgfsys@transformshift{1.027842in}{3.823230in}%
\pgfsys@useobject{currentmarker}{}%
\end{pgfscope}%
\end{pgfscope}%
\begin{pgfscope}%
\definecolor{textcolor}{rgb}{0.000000,0.000000,0.000000}%
\pgfsetstrokecolor{textcolor}%
\pgfsetfillcolor{textcolor}%
\pgftext[x=0.535138in, y=3.739897in, left, base]{\color{textcolor}\rmfamily\fontsize{16.000000}{19.200000}\selectfont \(\displaystyle {0.85}\)}%
\end{pgfscope}%
\begin{pgfscope}%
\pgfsetbuttcap%
\pgfsetroundjoin%
\definecolor{currentfill}{rgb}{0.000000,0.000000,0.000000}%
\pgfsetfillcolor{currentfill}%
\pgfsetlinewidth{0.803000pt}%
\definecolor{currentstroke}{rgb}{0.000000,0.000000,0.000000}%
\pgfsetstrokecolor{currentstroke}%
\pgfsetdash{}{0pt}%
\pgfsys@defobject{currentmarker}{\pgfqpoint{-0.048611in}{0.000000in}}{\pgfqpoint{-0.000000in}{0.000000in}}{%
\pgfpathmoveto{\pgfqpoint{-0.000000in}{0.000000in}}%
\pgfpathlineto{\pgfqpoint{-0.048611in}{0.000000in}}%
\pgfusepath{stroke,fill}%
}%
\begin{pgfscope}%
\pgfsys@transformshift{1.027842in}{4.389976in}%
\pgfsys@useobject{currentmarker}{}%
\end{pgfscope}%
\end{pgfscope}%
\begin{pgfscope}%
\definecolor{textcolor}{rgb}{0.000000,0.000000,0.000000}%
\pgfsetstrokecolor{textcolor}%
\pgfsetfillcolor{textcolor}%
\pgftext[x=0.535138in, y=4.306643in, left, base]{\color{textcolor}\rmfamily\fontsize{16.000000}{19.200000}\selectfont \(\displaystyle {0.90}\)}%
\end{pgfscope}%
\begin{pgfscope}%
\definecolor{textcolor}{rgb}{0.000000,0.000000,0.000000}%
\pgfsetstrokecolor{textcolor}%
\pgfsetfillcolor{textcolor}%
\pgftext[x=0.479583in,y=2.689738in,,bottom,rotate=90.000000]{\color{textcolor}\rmfamily\fontsize{16.000000}{19.200000}\selectfont TLS(x)}%
\end{pgfscope}%
\begin{pgfscope}%
\pgfpathrectangle{\pgfqpoint{1.027842in}{0.819475in}}{\pgfqpoint{5.077124in}{3.740525in}}%
\pgfusepath{clip}%
\pgfsetrectcap%
\pgfsetroundjoin%
\pgfsetlinewidth{1.505625pt}%
\definecolor{currentstroke}{rgb}{0.121569,0.466667,0.705882}%
\pgfsetstrokecolor{currentstroke}%
\pgfsetdash{}{0pt}%
\pgfpathmoveto{\pgfqpoint{1.027842in}{4.389976in}}%
\pgfpathlineto{\pgfqpoint{1.600450in}{4.389976in}}%
\pgfpathlineto{\pgfqpoint{1.613175in}{4.333651in}}%
\pgfpathlineto{\pgfqpoint{1.625899in}{4.231124in}}%
\pgfpathlineto{\pgfqpoint{1.638624in}{4.133724in}}%
\pgfpathlineto{\pgfqpoint{1.651349in}{4.041451in}}%
\pgfpathlineto{\pgfqpoint{1.664073in}{3.954303in}}%
\pgfpathlineto{\pgfqpoint{1.676798in}{3.872282in}}%
\pgfpathlineto{\pgfqpoint{1.689523in}{3.795388in}}%
\pgfpathlineto{\pgfqpoint{1.702247in}{3.723619in}}%
\pgfpathlineto{\pgfqpoint{1.714972in}{3.656977in}}%
\pgfpathlineto{\pgfqpoint{1.727696in}{3.595461in}}%
\pgfpathlineto{\pgfqpoint{1.740421in}{3.539072in}}%
\pgfpathlineto{\pgfqpoint{1.753146in}{3.487809in}}%
\pgfpathlineto{\pgfqpoint{1.765870in}{3.441672in}}%
\pgfpathlineto{\pgfqpoint{1.778595in}{3.400661in}}%
\pgfpathlineto{\pgfqpoint{1.791319in}{3.364777in}}%
\pgfpathlineto{\pgfqpoint{1.804044in}{3.334019in}}%
\pgfpathlineto{\pgfqpoint{1.816769in}{3.308388in}}%
\pgfpathlineto{\pgfqpoint{1.829493in}{3.287882in}}%
\pgfpathlineto{\pgfqpoint{1.842218in}{3.272503in}}%
\pgfpathlineto{\pgfqpoint{1.854943in}{3.262251in}}%
\pgfpathlineto{\pgfqpoint{1.867667in}{3.257125in}}%
\pgfpathlineto{\pgfqpoint{1.880392in}{3.257125in}}%
\pgfpathlineto{\pgfqpoint{1.893116in}{3.262251in}}%
\pgfpathlineto{\pgfqpoint{1.905841in}{3.272503in}}%
\pgfpathlineto{\pgfqpoint{1.918566in}{3.287882in}}%
\pgfpathlineto{\pgfqpoint{1.931290in}{3.308388in}}%
\pgfpathlineto{\pgfqpoint{1.944015in}{3.334019in}}%
\pgfpathlineto{\pgfqpoint{1.956740in}{3.364777in}}%
\pgfpathlineto{\pgfqpoint{1.969464in}{3.400661in}}%
\pgfpathlineto{\pgfqpoint{1.982189in}{3.441672in}}%
\pgfpathlineto{\pgfqpoint{1.994913in}{3.487809in}}%
\pgfpathlineto{\pgfqpoint{2.007638in}{3.539072in}}%
\pgfpathlineto{\pgfqpoint{2.020363in}{3.595461in}}%
\pgfpathlineto{\pgfqpoint{2.033087in}{3.627084in}}%
\pgfpathlineto{\pgfqpoint{2.045812in}{3.589918in}}%
\pgfpathlineto{\pgfqpoint{2.058537in}{3.563005in}}%
\pgfpathlineto{\pgfqpoint{2.071261in}{3.546345in}}%
\pgfpathlineto{\pgfqpoint{2.083986in}{3.539937in}}%
\pgfpathlineto{\pgfqpoint{2.096710in}{3.543782in}}%
\pgfpathlineto{\pgfqpoint{2.109435in}{3.557879in}}%
\pgfpathlineto{\pgfqpoint{2.122160in}{3.582229in}}%
\pgfpathlineto{\pgfqpoint{2.134884in}{3.616832in}}%
\pgfpathlineto{\pgfqpoint{2.147609in}{3.610360in}}%
\pgfpathlineto{\pgfqpoint{2.160333in}{3.552689in}}%
\pgfpathlineto{\pgfqpoint{2.173058in}{3.500144in}}%
\pgfpathlineto{\pgfqpoint{2.185783in}{3.452725in}}%
\pgfpathlineto{\pgfqpoint{2.198507in}{3.410433in}}%
\pgfpathlineto{\pgfqpoint{2.211232in}{3.373268in}}%
\pgfpathlineto{\pgfqpoint{2.223957in}{3.341228in}}%
\pgfpathlineto{\pgfqpoint{2.236681in}{3.314315in}}%
\pgfpathlineto{\pgfqpoint{2.249406in}{3.292528in}}%
\pgfpathlineto{\pgfqpoint{2.262130in}{3.275868in}}%
\pgfpathlineto{\pgfqpoint{2.274855in}{3.264333in}}%
\pgfpathlineto{\pgfqpoint{2.287580in}{3.222721in}}%
\pgfpathlineto{\pgfqpoint{2.300304in}{3.117887in}}%
\pgfpathlineto{\pgfqpoint{2.313029in}{3.023307in}}%
\pgfpathlineto{\pgfqpoint{2.325754in}{2.938979in}}%
\pgfpathlineto{\pgfqpoint{2.338478in}{2.864904in}}%
\pgfpathlineto{\pgfqpoint{2.351203in}{2.801081in}}%
\pgfpathlineto{\pgfqpoint{2.363927in}{2.710715in}}%
\pgfpathlineto{\pgfqpoint{2.376652in}{2.491892in}}%
\pgfpathlineto{\pgfqpoint{2.389377in}{2.258721in}}%
\pgfpathlineto{\pgfqpoint{2.402101in}{2.046056in}}%
\pgfpathlineto{\pgfqpoint{2.414826in}{1.853896in}}%
\pgfpathlineto{\pgfqpoint{2.427551in}{1.682242in}}%
\pgfpathlineto{\pgfqpoint{2.440275in}{1.531092in}}%
\pgfpathlineto{\pgfqpoint{2.453000in}{1.400448in}}%
\pgfpathlineto{\pgfqpoint{2.465724in}{1.290309in}}%
\pgfpathlineto{\pgfqpoint{2.478449in}{1.200675in}}%
\pgfpathlineto{\pgfqpoint{2.491174in}{1.131547in}}%
\pgfpathlineto{\pgfqpoint{2.503898in}{1.082924in}}%
\pgfpathlineto{\pgfqpoint{2.516623in}{1.054806in}}%
\pgfpathlineto{\pgfqpoint{2.529347in}{1.047194in}}%
\pgfpathlineto{\pgfqpoint{2.542072in}{1.060086in}}%
\pgfpathlineto{\pgfqpoint{2.554797in}{1.093484in}}%
\pgfpathlineto{\pgfqpoint{2.567521in}{1.123454in}}%
\pgfpathlineto{\pgfqpoint{2.580246in}{1.086365in}}%
\pgfpathlineto{\pgfqpoint{2.592971in}{1.064655in}}%
\pgfpathlineto{\pgfqpoint{2.605695in}{1.058324in}}%
\pgfpathlineto{\pgfqpoint{2.618420in}{1.067372in}}%
\pgfpathlineto{\pgfqpoint{2.631144in}{1.091799in}}%
\pgfpathlineto{\pgfqpoint{2.643869in}{1.131605in}}%
\pgfpathlineto{\pgfqpoint{2.656594in}{1.186790in}}%
\pgfpathlineto{\pgfqpoint{2.669318in}{1.257353in}}%
\pgfpathlineto{\pgfqpoint{2.682043in}{1.343296in}}%
\pgfpathlineto{\pgfqpoint{2.694768in}{1.444618in}}%
\pgfpathlineto{\pgfqpoint{2.707492in}{1.561318in}}%
\pgfpathlineto{\pgfqpoint{2.720217in}{1.693398in}}%
\pgfpathlineto{\pgfqpoint{2.732941in}{1.840856in}}%
\pgfpathlineto{\pgfqpoint{2.745666in}{2.003694in}}%
\pgfpathlineto{\pgfqpoint{2.758391in}{2.181910in}}%
\pgfpathlineto{\pgfqpoint{2.771115in}{2.375505in}}%
\pgfpathlineto{\pgfqpoint{2.783840in}{2.584479in}}%
\pgfpathlineto{\pgfqpoint{2.796565in}{2.808833in}}%
\pgfpathlineto{\pgfqpoint{2.809289in}{3.048565in}}%
\pgfpathlineto{\pgfqpoint{2.822014in}{3.274289in}}%
\pgfpathlineto{\pgfqpoint{2.834738in}{3.433026in}}%
\pgfpathlineto{\pgfqpoint{2.847463in}{3.602015in}}%
\pgfpathlineto{\pgfqpoint{2.860188in}{3.781256in}}%
\pgfpathlineto{\pgfqpoint{2.872912in}{3.970750in}}%
\pgfpathlineto{\pgfqpoint{2.885637in}{4.170497in}}%
\pgfpathlineto{\pgfqpoint{2.898361in}{4.349472in}}%
\pgfpathlineto{\pgfqpoint{2.911086in}{4.389976in}}%
\pgfpathlineto{\pgfqpoint{3.292825in}{4.389976in}}%
\pgfpathlineto{\pgfqpoint{3.305549in}{4.333651in}}%
\pgfpathlineto{\pgfqpoint{3.318274in}{4.231124in}}%
\pgfpathlineto{\pgfqpoint{3.330999in}{4.133724in}}%
\pgfpathlineto{\pgfqpoint{3.343723in}{4.041451in}}%
\pgfpathlineto{\pgfqpoint{3.356448in}{3.954303in}}%
\pgfpathlineto{\pgfqpoint{3.369172in}{3.872282in}}%
\pgfpathlineto{\pgfqpoint{3.381897in}{3.795388in}}%
\pgfpathlineto{\pgfqpoint{3.394622in}{3.723619in}}%
\pgfpathlineto{\pgfqpoint{3.407346in}{3.656977in}}%
\pgfpathlineto{\pgfqpoint{3.420071in}{3.595461in}}%
\pgfpathlineto{\pgfqpoint{3.432796in}{3.539072in}}%
\pgfpathlineto{\pgfqpoint{3.445520in}{3.487809in}}%
\pgfpathlineto{\pgfqpoint{3.458245in}{3.441672in}}%
\pgfpathlineto{\pgfqpoint{3.470969in}{3.400661in}}%
\pgfpathlineto{\pgfqpoint{3.483694in}{3.364777in}}%
\pgfpathlineto{\pgfqpoint{3.496419in}{3.334019in}}%
\pgfpathlineto{\pgfqpoint{3.509143in}{3.308388in}}%
\pgfpathlineto{\pgfqpoint{3.521868in}{3.287882in}}%
\pgfpathlineto{\pgfqpoint{3.534593in}{3.272503in}}%
\pgfpathlineto{\pgfqpoint{3.547317in}{3.262251in}}%
\pgfpathlineto{\pgfqpoint{3.560042in}{3.257125in}}%
\pgfpathlineto{\pgfqpoint{3.572766in}{3.257125in}}%
\pgfpathlineto{\pgfqpoint{3.585491in}{3.262251in}}%
\pgfpathlineto{\pgfqpoint{3.598216in}{3.272503in}}%
\pgfpathlineto{\pgfqpoint{3.610940in}{3.287882in}}%
\pgfpathlineto{\pgfqpoint{3.623665in}{3.308388in}}%
\pgfpathlineto{\pgfqpoint{3.636389in}{3.334019in}}%
\pgfpathlineto{\pgfqpoint{3.649114in}{3.364777in}}%
\pgfpathlineto{\pgfqpoint{3.661839in}{3.400661in}}%
\pgfpathlineto{\pgfqpoint{3.674563in}{3.441672in}}%
\pgfpathlineto{\pgfqpoint{3.687288in}{3.487809in}}%
\pgfpathlineto{\pgfqpoint{3.700013in}{3.539072in}}%
\pgfpathlineto{\pgfqpoint{3.712737in}{3.595461in}}%
\pgfpathlineto{\pgfqpoint{3.725462in}{3.656977in}}%
\pgfpathlineto{\pgfqpoint{3.738186in}{3.723619in}}%
\pgfpathlineto{\pgfqpoint{3.750911in}{3.795388in}}%
\pgfpathlineto{\pgfqpoint{3.763636in}{3.872282in}}%
\pgfpathlineto{\pgfqpoint{3.776360in}{3.954303in}}%
\pgfpathlineto{\pgfqpoint{3.789085in}{4.041451in}}%
\pgfpathlineto{\pgfqpoint{3.801810in}{4.133724in}}%
\pgfpathlineto{\pgfqpoint{3.814534in}{4.231124in}}%
\pgfpathlineto{\pgfqpoint{3.827259in}{4.333651in}}%
\pgfpathlineto{\pgfqpoint{3.839983in}{4.389976in}}%
\pgfpathlineto{\pgfqpoint{4.985199in}{4.389976in}}%
\pgfpathlineto{\pgfqpoint{4.997924in}{4.333651in}}%
\pgfpathlineto{\pgfqpoint{5.010649in}{4.231124in}}%
\pgfpathlineto{\pgfqpoint{5.023373in}{4.133724in}}%
\pgfpathlineto{\pgfqpoint{5.036098in}{4.041451in}}%
\pgfpathlineto{\pgfqpoint{5.048822in}{3.954303in}}%
\pgfpathlineto{\pgfqpoint{5.061547in}{3.872282in}}%
\pgfpathlineto{\pgfqpoint{5.074272in}{3.795388in}}%
\pgfpathlineto{\pgfqpoint{5.112445in}{3.306331in}}%
\pgfpathlineto{\pgfqpoint{5.125170in}{3.159462in}}%
\pgfpathlineto{\pgfqpoint{5.137895in}{3.022845in}}%
\pgfpathlineto{\pgfqpoint{5.150619in}{2.896482in}}%
\pgfpathlineto{\pgfqpoint{5.163344in}{2.780371in}}%
\pgfpathlineto{\pgfqpoint{5.176069in}{2.674512in}}%
\pgfpathlineto{\pgfqpoint{5.188793in}{2.578907in}}%
\pgfpathlineto{\pgfqpoint{5.201518in}{2.493553in}}%
\pgfpathlineto{\pgfqpoint{5.214242in}{2.418453in}}%
\pgfpathlineto{\pgfqpoint{5.226967in}{2.353605in}}%
\pgfpathlineto{\pgfqpoint{5.239692in}{2.299010in}}%
\pgfpathlineto{\pgfqpoint{5.252416in}{2.254667in}}%
\pgfpathlineto{\pgfqpoint{5.265141in}{2.220577in}}%
\pgfpathlineto{\pgfqpoint{5.277866in}{2.196740in}}%
\pgfpathlineto{\pgfqpoint{5.290590in}{2.183155in}}%
\pgfpathlineto{\pgfqpoint{5.303315in}{2.179823in}}%
\pgfpathlineto{\pgfqpoint{5.316039in}{2.186743in}}%
\pgfpathlineto{\pgfqpoint{5.328764in}{2.203917in}}%
\pgfpathlineto{\pgfqpoint{5.341489in}{2.231342in}}%
\pgfpathlineto{\pgfqpoint{5.354213in}{2.269021in}}%
\pgfpathlineto{\pgfqpoint{5.366938in}{2.316952in}}%
\pgfpathlineto{\pgfqpoint{5.379663in}{2.375136in}}%
\pgfpathlineto{\pgfqpoint{5.392387in}{2.443572in}}%
\pgfpathlineto{\pgfqpoint{5.405112in}{2.522261in}}%
\pgfpathlineto{\pgfqpoint{5.417836in}{2.611202in}}%
\pgfpathlineto{\pgfqpoint{5.430561in}{2.710397in}}%
\pgfpathlineto{\pgfqpoint{5.443286in}{2.819843in}}%
\pgfpathlineto{\pgfqpoint{5.456010in}{2.939543in}}%
\pgfpathlineto{\pgfqpoint{5.468735in}{3.069495in}}%
\pgfpathlineto{\pgfqpoint{5.481459in}{3.209700in}}%
\pgfpathlineto{\pgfqpoint{5.494184in}{3.360157in}}%
\pgfpathlineto{\pgfqpoint{5.506909in}{3.520867in}}%
\pgfpathlineto{\pgfqpoint{5.519633in}{3.691830in}}%
\pgfpathlineto{\pgfqpoint{5.532358in}{3.821718in}}%
\pgfpathlineto{\pgfqpoint{5.545083in}{3.900407in}}%
\pgfpathlineto{\pgfqpoint{5.557807in}{3.984222in}}%
\pgfpathlineto{\pgfqpoint{5.570532in}{4.073163in}}%
\pgfpathlineto{\pgfqpoint{5.583256in}{4.167231in}}%
\pgfpathlineto{\pgfqpoint{5.595981in}{4.266426in}}%
\pgfpathlineto{\pgfqpoint{5.608706in}{4.370746in}}%
\pgfpathlineto{\pgfqpoint{5.621430in}{4.389976in}}%
\pgfpathlineto{\pgfqpoint{6.104966in}{4.389976in}}%
\pgfpathlineto{\pgfqpoint{6.104966in}{4.389976in}}%
\pgfusepath{stroke}%
\end{pgfscope}%
\begin{pgfscope}%
\pgfpathrectangle{\pgfqpoint{1.027842in}{0.819475in}}{\pgfqpoint{5.077124in}{3.740525in}}%
\pgfusepath{clip}%
\pgfsetrectcap%
\pgfsetroundjoin%
\pgfsetlinewidth{1.505625pt}%
\definecolor{currentstroke}{rgb}{0.000000,0.000000,0.000000}%
\pgfsetstrokecolor{currentstroke}%
\pgfsetstrokeopacity{0.500000}%
\pgfsetdash{}{0pt}%
\pgfusepath{stroke}%
\end{pgfscope}%
\begin{pgfscope}%
\pgfpathrectangle{\pgfqpoint{1.027842in}{0.819475in}}{\pgfqpoint{5.077124in}{3.740525in}}%
\pgfusepath{clip}%
\pgfsetrectcap%
\pgfsetroundjoin%
\pgfsetlinewidth{1.505625pt}%
\definecolor{currentstroke}{rgb}{0.000000,0.000000,0.000000}%
\pgfsetstrokecolor{currentstroke}%
\pgfsetstrokeopacity{0.500000}%
\pgfsetdash{}{0pt}%
\pgfpathmoveto{\pgfqpoint{1.874030in}{0.989499in}}%
\pgfpathlineto{\pgfqpoint{1.874030in}{4.389976in}}%
\pgfusepath{stroke}%
\end{pgfscope}%
\begin{pgfscope}%
\pgfpathrectangle{\pgfqpoint{1.027842in}{0.819475in}}{\pgfqpoint{5.077124in}{3.740525in}}%
\pgfusepath{clip}%
\pgfsetrectcap%
\pgfsetroundjoin%
\pgfsetlinewidth{1.505625pt}%
\definecolor{currentstroke}{rgb}{0.000000,0.000000,0.000000}%
\pgfsetstrokecolor{currentstroke}%
\pgfsetstrokeopacity{0.500000}%
\pgfsetdash{}{0pt}%
\pgfpathmoveto{\pgfqpoint{2.297123in}{0.989499in}}%
\pgfpathlineto{\pgfqpoint{2.297123in}{4.389976in}}%
\pgfusepath{stroke}%
\end{pgfscope}%
\begin{pgfscope}%
\pgfpathrectangle{\pgfqpoint{1.027842in}{0.819475in}}{\pgfqpoint{5.077124in}{3.740525in}}%
\pgfusepath{clip}%
\pgfsetrectcap%
\pgfsetroundjoin%
\pgfsetlinewidth{1.505625pt}%
\definecolor{currentstroke}{rgb}{0.000000,0.000000,0.000000}%
\pgfsetstrokecolor{currentstroke}%
\pgfsetstrokeopacity{0.500000}%
\pgfsetdash{}{0pt}%
\pgfpathmoveto{\pgfqpoint{2.550979in}{0.989499in}}%
\pgfpathlineto{\pgfqpoint{2.550979in}{4.389976in}}%
\pgfusepath{stroke}%
\end{pgfscope}%
\begin{pgfscope}%
\pgfpathrectangle{\pgfqpoint{1.027842in}{0.819475in}}{\pgfqpoint{5.077124in}{3.740525in}}%
\pgfusepath{clip}%
\pgfsetrectcap%
\pgfsetroundjoin%
\pgfsetlinewidth{1.505625pt}%
\definecolor{currentstroke}{rgb}{0.000000,0.000000,0.000000}%
\pgfsetstrokecolor{currentstroke}%
\pgfsetstrokeopacity{0.500000}%
\pgfsetdash{}{0pt}%
\pgfpathmoveto{\pgfqpoint{2.635598in}{0.989499in}}%
\pgfpathlineto{\pgfqpoint{2.635598in}{4.389976in}}%
\pgfusepath{stroke}%
\end{pgfscope}%
\begin{pgfscope}%
\pgfpathrectangle{\pgfqpoint{1.027842in}{0.819475in}}{\pgfqpoint{5.077124in}{3.740525in}}%
\pgfusepath{clip}%
\pgfsetrectcap%
\pgfsetroundjoin%
\pgfsetlinewidth{1.505625pt}%
\definecolor{currentstroke}{rgb}{0.000000,0.000000,0.000000}%
\pgfsetstrokecolor{currentstroke}%
\pgfsetstrokeopacity{0.500000}%
\pgfsetdash{}{0pt}%
\pgfpathmoveto{\pgfqpoint{2.627136in}{0.989499in}}%
\pgfpathlineto{\pgfqpoint{2.627136in}{4.389976in}}%
\pgfusepath{stroke}%
\end{pgfscope}%
\begin{pgfscope}%
\pgfpathrectangle{\pgfqpoint{1.027842in}{0.819475in}}{\pgfqpoint{5.077124in}{3.740525in}}%
\pgfusepath{clip}%
\pgfsetrectcap%
\pgfsetroundjoin%
\pgfsetlinewidth{1.505625pt}%
\definecolor{currentstroke}{rgb}{0.000000,0.000000,0.000000}%
\pgfsetstrokecolor{currentstroke}%
\pgfsetstrokeopacity{0.500000}%
\pgfsetdash{}{0pt}%
\pgfpathmoveto{\pgfqpoint{3.566404in}{0.989499in}}%
\pgfpathlineto{\pgfqpoint{3.566404in}{4.389976in}}%
\pgfusepath{stroke}%
\end{pgfscope}%
\begin{pgfscope}%
\pgfpathrectangle{\pgfqpoint{1.027842in}{0.819475in}}{\pgfqpoint{5.077124in}{3.740525in}}%
\pgfusepath{clip}%
\pgfsetrectcap%
\pgfsetroundjoin%
\pgfsetlinewidth{1.505625pt}%
\definecolor{currentstroke}{rgb}{0.000000,0.000000,0.000000}%
\pgfsetstrokecolor{currentstroke}%
\pgfsetstrokeopacity{0.500000}%
\pgfsetdash{}{0pt}%
\pgfpathmoveto{\pgfqpoint{5.258779in}{0.989499in}}%
\pgfpathlineto{\pgfqpoint{5.258779in}{4.389976in}}%
\pgfusepath{stroke}%
\end{pgfscope}%
\begin{pgfscope}%
\pgfpathrectangle{\pgfqpoint{1.027842in}{0.819475in}}{\pgfqpoint{5.077124in}{3.740525in}}%
\pgfusepath{clip}%
\pgfsetrectcap%
\pgfsetroundjoin%
\pgfsetlinewidth{1.505625pt}%
\definecolor{currentstroke}{rgb}{0.000000,0.000000,0.000000}%
\pgfsetstrokecolor{currentstroke}%
\pgfsetstrokeopacity{0.500000}%
\pgfsetdash{}{0pt}%
\pgfpathmoveto{\pgfqpoint{5.343397in}{0.989499in}}%
\pgfpathlineto{\pgfqpoint{5.343397in}{4.389976in}}%
\pgfusepath{stroke}%
\end{pgfscope}%
\begin{pgfscope}%
\pgfsetrectcap%
\pgfsetmiterjoin%
\pgfsetlinewidth{0.803000pt}%
\definecolor{currentstroke}{rgb}{0.000000,0.000000,0.000000}%
\pgfsetstrokecolor{currentstroke}%
\pgfsetdash{}{0pt}%
\pgfpathmoveto{\pgfqpoint{1.027842in}{0.819475in}}%
\pgfpathlineto{\pgfqpoint{1.027842in}{4.560000in}}%
\pgfusepath{stroke}%
\end{pgfscope}%
\begin{pgfscope}%
\pgfsetrectcap%
\pgfsetmiterjoin%
\pgfsetlinewidth{0.803000pt}%
\definecolor{currentstroke}{rgb}{0.000000,0.000000,0.000000}%
\pgfsetstrokecolor{currentstroke}%
\pgfsetdash{}{0pt}%
\pgfpathmoveto{\pgfqpoint{6.104966in}{0.819475in}}%
\pgfpathlineto{\pgfqpoint{6.104966in}{4.560000in}}%
\pgfusepath{stroke}%
\end{pgfscope}%
\begin{pgfscope}%
\pgfsetrectcap%
\pgfsetmiterjoin%
\pgfsetlinewidth{0.803000pt}%
\definecolor{currentstroke}{rgb}{0.000000,0.000000,0.000000}%
\pgfsetstrokecolor{currentstroke}%
\pgfsetdash{}{0pt}%
\pgfpathmoveto{\pgfqpoint{1.027842in}{0.819475in}}%
\pgfpathlineto{\pgfqpoint{6.104966in}{0.819475in}}%
\pgfusepath{stroke}%
\end{pgfscope}%
\begin{pgfscope}%
\pgfsetrectcap%
\pgfsetmiterjoin%
\pgfsetlinewidth{0.803000pt}%
\definecolor{currentstroke}{rgb}{0.000000,0.000000,0.000000}%
\pgfsetstrokecolor{currentstroke}%
\pgfsetdash{}{0pt}%
\pgfpathmoveto{\pgfqpoint{1.027842in}{4.560000in}}%
\pgfpathlineto{\pgfqpoint{6.104966in}{4.560000in}}%
\pgfusepath{stroke}%
\end{pgfscope}%
\end{pgfpicture}%
\makeatother%
\endgroup%

%% file: sec/contractor.tex
\subsection{Contracting the search space}  

To greatly improve the efficiency of the search over rotations and translations, we use a \textit{contractor}. 
The contractor is an algorithm that finds a subset of the search space that still provably contains the global optimum. 


The key observation is that the TLS objective implicitly has a feasible set. Consider the set of 3D transforms for which the $i$-th residual is not truncated:

\begin{equation}
	\begin{aligned}
		\label{eq:voting-set}
		\mathcal{V}_i := \{(\MR, \vt) \in \mathcal{R} \times \Rthree : \normsq{\MR \vp_i  - \vq_i  + \vt} \leq \epssq \}
	\end{aligned}
\end{equation}
\begin{definition}	
	The union of all $\mathcal{V}_i$ is the feasible set $\mathcal{F}$ for the TLS problem (\ref{eq:pcr-tls})
		\begin{equation}
			\begin{aligned}
				\mathcal{F} = \bigcup_{i=1}^N \mathcal{V}_i\\ 
			\end{aligned}
		\end{equation}	
\end{definition}

\begin{lemma}
Assuming there exists any 3D transform for which any residual of the TLS objective (\ref{eq:pcr-tls}) is not truncated, the global minimizer $(\MR^*, \vt^*)$ of problem \ref{eq:pcr-tls} is contained in the feasible set,
\begin{equation}
	\begin{aligned}
		(\MR^*, \vt^*) \in \mathcal{F}
	\end{aligned}
\end{equation}	
\end{lemma}

\begin{proof}
	If there exists any 3D transform for which any residual of the TLS objective is not truncated, then at the global minimizer this residual is not truncated, otherwise it would not be a minimizer. From this follows, that the minimizer is contained it at least one $\mathcal{V}_i$ and therefore in $\mathcal{F}$.
\end{proof} 
This assumption is justified since otherwise the objective is constant and therefore uninteresting for optimization.

Formally, the contractor is an algorithm that contracts a given a search space $C$ to the smaller space $C^c \subseteq C$ that still has the same intersection with the feasible set, $C^c \cap \mathcal{F} = C \cap \mathcal{F}$ \cite{CHABERT20091079}. This property guarantees that if the original space $C$ contains the global optimum, so does the contracted space $C^c$. The smaller this contracted space is, the more efficiency is gained during BnB. 


\textbf{Contracting \SOtwo}. In the following we present a contractor that contracts the space of 2D rotation $\SOtwo$. 
Given is a BnB-node $B_t \times \SOtwo \subseteq \Rthree \times \SOtwo$ where $B_t$ is a ball of translations. It contains initially the full set of 2D rotations. We contract this space to the space $B_t \times B_r$, i.e. we contract the rotation space to the ball $B_r$.

\begin{theorem}
	\label{thm:so2-contractor}
Given is a subspace $B_t \times \SOtwo \subseteq \Rthree \times \SOtwo$ where $B_t = \{ \vt \in \Rthree : || \vt_c - \vt || \leq t_r \}$ is a ball of translations, then $B_r$ is a contracted ball of rotations so that the \textit{contractor completeness} property
\begin{equation}
	\begin{aligned}
		B_t \times \SOtwo \cap \mathcal{F} = B_t \times B_r \cap \mathcal{F}
	\end{aligned}
\end{equation}	
holds. $B_r$ is defined as
\begin{equation}
	\begin{aligned}
		\label{eq:rotation-balls-given-t-ball}
		B_r &= \bigcup_{i=1}^N a^i\\
	\end{aligned}
\end{equation}
where $a^i$ are intervals on the unit circle defined as
\begin{equation}
	\begin{aligned}
		a^i &= \{ \MR \in \SOtwo : d_{\angle}(\MR, R_c^i) \leq r_r^i\}\\
		R_c^i &= \angle(\vp_i, \vb_i)\\
		r_r^i &= \arccos \left( \min\{1, \max\{ -1, h_i \} \} \right)\\
		h_i &= \frac{\normsq{\vp_i} + \normsq{\vb_i} - (\epsilon + t_r)^2}{2 \norm{\vp_i} \norm{\vb_i}} \\
		\vb_i &= \vq_i - \vt_c\\
	\end{aligned}
\end{equation}
\end{theorem}

\textbf{Remark}. $d_{\angle}(\cdot, \cdot)$ is the geodesic (angular) distance between two rotations (2D or 3D), regardless of representation: In this case, $R_c^i$ is a 2D rotation represented by a scalar angle. 

\textbf{Improved contractor based on lower bound}.
The disadvantage of taking the union of all $a^i$'s as the contracted space (Eq. \ref{eq:rotation-balls-given-t-ball}) however is that only small contractions can be achieved. 
A drastic improvement is to consider the best solution found so far, i.e. the upper bound on the TLS objective. From this upper bound $\mathrm{UB}_{\mathrm{tls}}$ follows a lower bound on the number of non-trucated residuals $\mathrm{LB}_{\mathrm{mc}}$:

\begin{equation}
	\begin{aligned}
		\mathrm{LB}_{\mathrm{mc}} = \left\lceil N -  \frac{\mathrm{UB}_{\mathrm{tls}}}{\epssq} \right\rceil
	\end{aligned}
\end{equation}

Using this lower bound, we compute the smallest rotation interval that contains only points where at least $\mathrm{LB}_{\mathrm{mc}}$ many of the given intervals $a^i$ intersect.

We find this smallest rotation interval very efficiently using a sweep-line algorithm that requires $\Onlogn$ time for sorting the intervals $a^i$'s.

%% file: sec/4_evaluation.tex
\section{Experimental evaluation}
\label{sec:evaluation}
In this section we evaluate the runtime and accuracy of our Branch and Bound solver \MyAlgo{}. We evaluate two different problems: \code{pose-so2}, that is problem (\ref{eq:pcr-tls}), and for comparison with state of the art, \code{rotation}, the Wahba problem over $\SOthree$. In the \code{pose-so2} problem, we register two 3D point clouds but have to provide the correct rotation axis.
We performed the experiments on a desktop PC with an Intel i7 12700KF CPU and 32 GB of RAM.
Our algorithm is implemented in C++.

\subsection{Synthetic experiments}

We generate synthetic instances of the point cloud registration problem by first sampling inlier points from a unit cube $s \cdot [-1, 1]^3$ with scale $s$. Then, we  uniformly sample a perturbation for each point from a ball with radius $\epsilon^2$. Finally, we  uniformly sample outliers from the bounding box of the inliers. Throughout all experiments, we use $s=10$ and $\epsilon=0.5$. 
For the ground-truth transform, we uniformly sample $\SOthree$ for the rotation. For the translation, we uniformly sample a ball with radius $s$. We provide the solver with the ground-truth axis of rotation $\vn^*$.

\subsubsection{Runtime and scalability} 

\begin{figure}[!ht]
	\centering
	\includegraphics[width=0.9\linewidth]{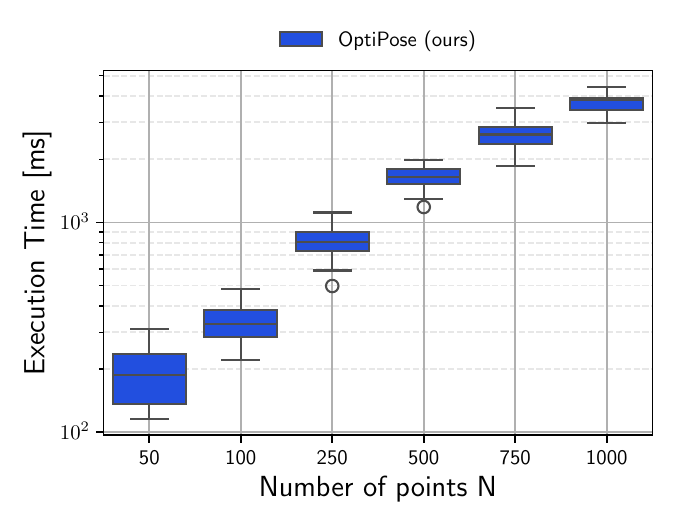}
	\caption{Our globally optimal solver scales well on 3D point cloud registration with known axis (synthetic data, 50\% outlier rate).}
	\label{fig:rotation-exe-time-pose}
\end{figure}

We first evaluate the scalability of the 3D point cloud registration with fixed axis (\code{pose-so2} problem).  As shown in Fig. \ref{fig:rotation-exe-time-pose}, our solver is able to register two point clouds with $N=100$ points in near real-time.

\subsubsection{Comparison with SDP relaxation solver STRIDE} 
In this section we compare our Branch and Bound solver with the state of the art SDP solver STRIDE \cite{9785843}. STRIDE is designed specifically to solve the large-scale problems that arise from tight semidefinite relaxations of the TLS problem. 
To our knowledge, STRIDE is currently the only solver that can solve the TLS registration problem with sizes as large as $N=100$ to proven global optimality. 

\textbf{Experimental procedure}. Our solver cannot solve the TLS registration problem over the full 3D transform as STRIDE does, and instead requires the axis of rotation (\code{pose-so2} problem). We therefore do the comparison only for the \code{rotation} problem.
We do the comparison on synthetic instances as generated by our procedure and measure the runtime, estimation errors and the achieved suboptimality. We use the public MATLAB-implementation of the authors \cite{cert-perception-github}.
We perform five trials for STRIDE and 20 for the proposed method OptiPose.

We define the suboptimality $\eta$ of the solution based on the upper bound (UB) and the lower bound (LB) (as in \cite{9785843}):
\begin{equation}
	\begin{aligned}    
		\eta = \frac{\text{UB} - \text{LB}}{1 + \text{UB} + \text{LB}}
	\end{aligned}
\end{equation}
The upper and lower bound are easily obtained from our BnB solver at termination. For details how STRIDE obtains its lower bound, see \cite{9785843}.
In case our BnB solver \textit{converges}, UB$=$LB holds, and the suboptimality is therefore zero. However, because we perform all computations in single precision floating point arithmetic, we assign in this case for $\eta$ a dummy value of $10^{-6}$.

\begin{figure}[!ht]
	\centering
	\includegraphics[width=0.9\linewidth]{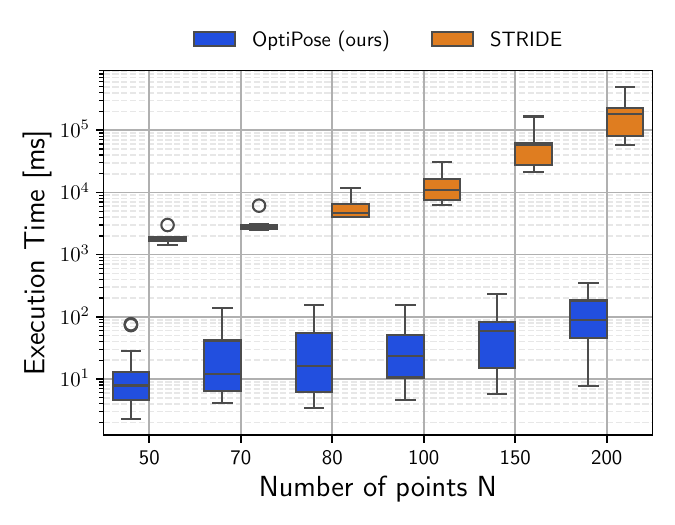}
	\caption{On average, our solver solves the rotation-only TLS problem in 100 milliseconds, whereas  the state-of-the-art solver STRIDE is two orders of magnitude slower (synthetic data, 50\% outlier rate).}
	\label{fig:rotation-exe-time-vs-stride}
\end{figure}
\begin{figure}[!ht]
	\centering
	\includegraphics[width=0.9\linewidth]{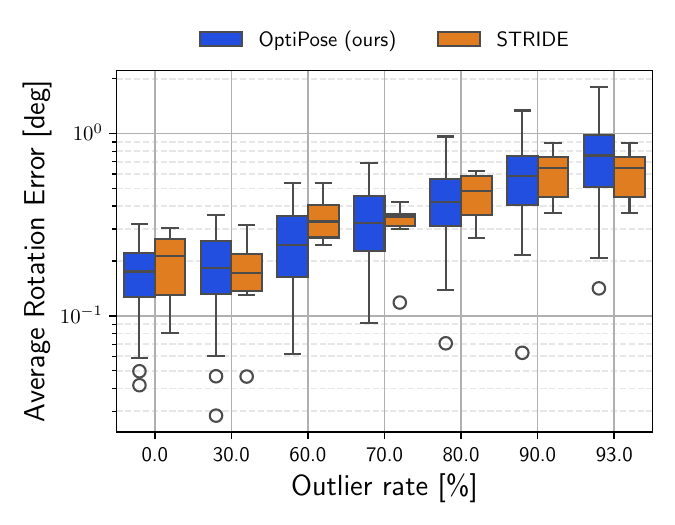}
	\caption{The rotation error compared to STRIDE on the rotation-only TLS problem. It remains below a median of one degree even at a very high outlier rates of 93\% (synthetic data, $N=100$).}
	\label{fig:rot-err-vs-stride}
\end{figure}
\textbf{Results}.
First, we measured the runtime for the rotation-only problem with $50\%$ outlier rate while varying the point cloud size. The results are shown in Fig. \ref{fig:rotation-exe-time-vs-stride}. 
Next, we set $N=100$ and varied the outlier rate to evaluate the achieved rotation error. As shown in Fig. \ref{fig:rot-err-vs-stride}, our solver achieves a similar estimation error to  STRIDE.

\begin{figure}[!tb]
	\centering
	\includegraphics[width=\linewidth]{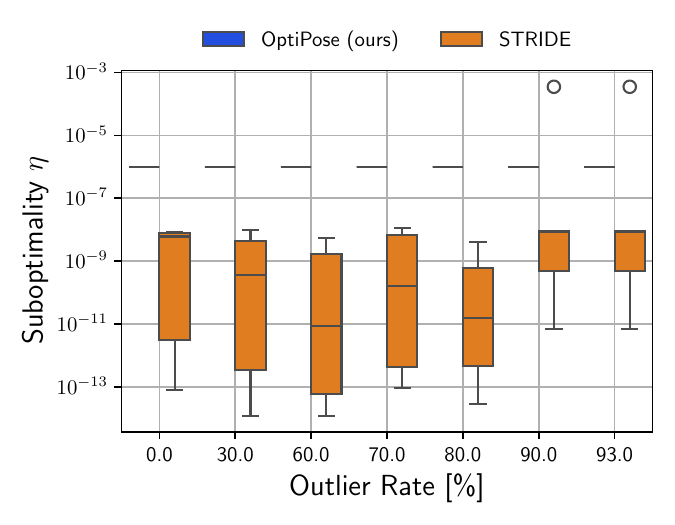}
	\caption{The reached $\eta$-suboptimality on synthetic data, rotation-only TLS problem, $N=30$. Lower is better. Our solver converged (proved global optimality) in all instances, yielding $\eta = 10^{-6}$. STRIDE on the other hand terminates with an increased suboptimality above $90\%$ outlier rate.}
	\label{fig:stride-subopt-eval}
\end{figure}

Next, we evaluated the achieved relative suboptimality over the outlier rate, the results are shown in Fig. \ref{fig:stride-subopt-eval}. 

\subsection{Adversarial experiments}  

In this section, we construct instances with a second good local minimum that is very close to the global minimum. This local minimum has a similar objective function value but a different minimizer.
The motivation is that, most instances of the TLS problem are easy because there is only one good global minimum and no other other good local minimas. Therefore, we create synthetic instances with the explicit aim to contain a second good local minimum to obtain empirical evidence that our solver is indeed globally optimal. We call such instances adversarial instances.

\textbf{Creating adversarial instances}. We create such adversarial instances by the following steps: 1) First, we sample a regular instance of the point cloud registration problem with $N$ points, we denote it as $(\MP_1, \MQ_2, \MT_{1}^*)$, $\MT_{1}^*$ is the ground-truth transform.  
Then, we  sample a second instance $(\MP_2, \MQ_2, \MT_{2}^*)$, with $a N$ many points, where $a$ is a factor between zero and one that controls how close (roughly) the second local minimum is to the global minimum. The higher $a$ is, the closer the local optimum is to the global optimum.
After this, we concatenate $\MP_1$ and $\MP_2$ to obtain the new $\MP$ point cloud, as well as  $\MQ_1$ and $\MQ_2$ to obtain the new $\MQ$ point cloud of the adversarial instance. This adversarial instance now has two good minimizers, namely $\MT_{1}^*$ and $\MT_{2}^*$. 
Finally, we simply evaluate the TLS objective at these two transforms to determine which one is the global minimizer.

\begin{figure}[!ht]
	\centering
	\includegraphics[width=0.49\linewidth]{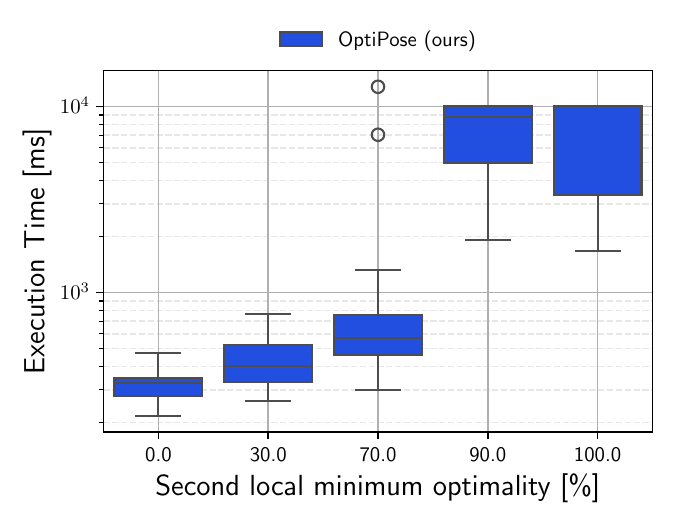}
	\includegraphics[width=0.49\linewidth]{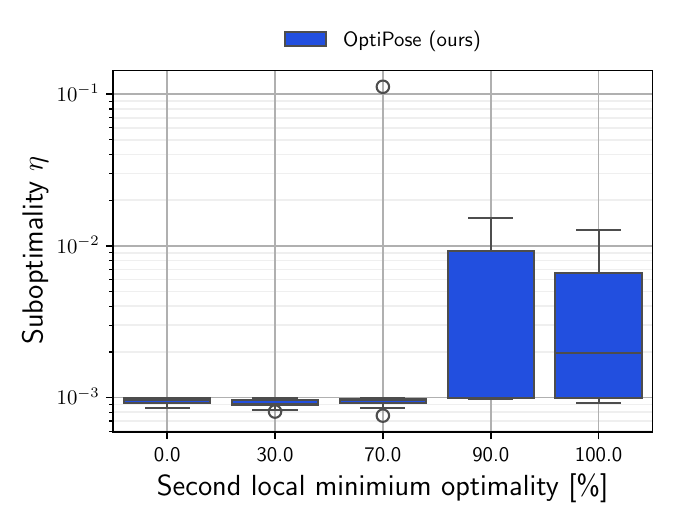}
	\caption{When there is a second local optimum close to the global optimum, proving global optimality takes more time, resulting in the time limit of 10s being hit more frequently and therefore an increased average $\eta$-suboptimality (\code{pose-so2} problem, $N=100$, $50\%$ outlier-rate).}
	\label{fig:runtime-adv}
\end{figure}

\textbf{Experimental procedure}. We solved the \code{pose-so2} problem, $N$ was set to 100 and the outlier rate at $50\%$.
We varied the factor of adversarial points $a$ from 0 to 1 and evaluated whether our solver either 1) found the global minimum, or 2) exceeded the time limit but output a valid lower bound. The second case means that BnB proved that a better solution might exist. For each value of $a$, we performed 100 trials. We set  a time limit of 10 seconds as well as a $\eta$-suboptimality threshold of $10^{-3}$ as termination criteria.

\textbf{Results}. We did not observe any single instance where our solver output a suboptimal solution, claiming it's the optimal one. This means that either it outputs the global optimum or it outputs a suboptimal solution and provides a valid lower bound.
However, we observed that our solver has more difficulty proving global optimality when there is a second good local minimum. This is indicated by an increasing runtime (Fig. \ref{fig:runtime-adv}). The solver also ran out of time (10 seconds) more frequently, resulting in an increased average $\eta$-suboptimality.

\subsection{Relaxation runtime}
\begin{figure}[!htb]
	\centering
	\includegraphics[width=\linewidth]{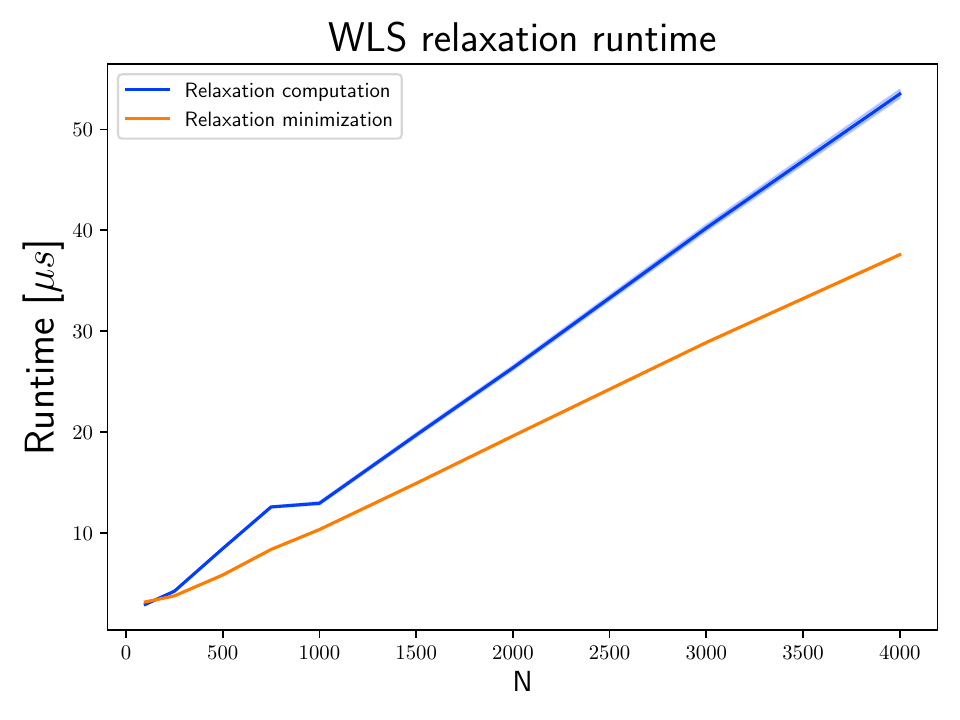}
	\caption{The runtime of our convex relaxation scales linearly with the data size $N$ and is in practice very fast, taking only $50$ microseconds for 4000 points.}
	\label{fig:wls-relax-runtime}
\end{figure}

The per iteration costs of our solver are kept at a minimum due to the very fast computation and minimization of our convex relaxation. As Fig. \ref{fig:wls-relax-runtime} shows, it scales for thousands of points and takes only 50 microseconds. 

\subsection{Discussion}

\textbf{Runtime and scalability}. 
The results demonstrate  that our solver scales well when solving the TLS registration problem \ref{eq:pcr-tls}. 
It can  register 100 points in 100ms, achieving real-time performance. 

\textbf{Comparison with SDP relaxation solver STRIDE}.
Compared to the state-of-the-art solver STRIDE, our solver can solve the rotation-only TLS problem two orders of magnitude faster. Our solver converged to the global minimum in all of the tested instances. STRIDE on the other hand can find and certify the global optimum only as long as the outlier rate is below 90\%, above which the $\eta$-suboptimality increases to $10^{-4}$.
The estimated rotation error is as low as with STRIDE, even at high outlier rates above $90\%$, it stays below one degree. 
These results suggest that our solver correctly solves the TLS problem, as it achieves the same level of robustness and low estimation errors expected from the TLS estimator.

One advantage of our Branch and Bound method is that it is able to prove global optimality even for extreme outlier rates above $90\%$. SDP relaxation approaches like STRIDE fail in such instances because the relaxation is not tight. 
In other words, while SDP relaxations fail on difficult instances, our method can solve all instances of the TLS problem, given  enough time.

\textbf{Adversarial experiments}. We performed adversarial experiments with a second local optimum close to the global optimum. They verified the correctness of the implementation. We did not observe any single instance where our solver computed an invalid lower bound at termination, indicating it is indeed globally optimal.
Additionally, during the development of the implementation, the adversarial instances proved a useful tool for uncovering implementation bugs.

\textbf{Relaxation runtime}. 
Convex relaxations are a well-known idea for global optimization. Generally, the more tight a relaxation is, the more expensive it is to compute. 

Prevalent SoTa methods use semidefinite convex relaxations that need to be computed and minimized only once to obtain a \textit{tight} lower bound. Our approach on the other hand is to use non-tight relaxations that need to be computed many times as part of Branch-and-Bound, but are however much cheaper to compute.The tighter a relaxation needs to be, the more computationally expensive it is.

%% file: sec/5_conclusion.tex
\section{Conclusion}
\label{sec:conclusion}

In this paper, we designed a fast and globally optimal solver for robust point cloud registration in the Truncated Least Squares (TLS) formulation. 
We showed that the TLS objective admits a convex relaxation that leads to simple weighted least-squares problems. These problems can be solved easily by adapting existing least-squares solvers.
We showed that by careful consideration of the feasible set of the TLS problem, we can greatly reduce the search space for BnB using a contractor method. Compared with state-of-the-art tight semidefinite relaxations, our method uses inexpensive to compute loose relaxations. This yields a lower average runtime than tight relaxations, even if they need to be computed hundreds of times.

\textbf{Limitations}. The main limitation of our solver is that it cannot solve for the full 3D transform and instead requires the axis of rotation.
While for autonomous vehicles, the axis of rotation is approximately the gravity direction which can be measured using Inertial Measurement Units (IMUs), a general approach would be favorable. 

\textbf{Future work}.
Future work will focus on generalizing the contractor method to 3D rotations so that the problem can be solved over the full 3D transform.
Another promising direction for future research is to study the trade-off between tightness and the computational cost of different convex relaxations.

%% file: sec/X_suppl.tex
\clearpage
\setcounter{page}{1}
\maketitlesupplementary

\begin{appendix}

\section{WLS relaxation}
\label{proof:wls-relax}

Proof of Theorem \ref{thm:wls-relax}.
\begin{proof}
	
	In the TLS objective, non-convexity is introduced solely by the truncation since the squared terms are convex.
	Therefore, we only need to show that the re-weighted sum leads to the residuals not being truncated. 
	For this, we apply an affine transform on each residual so that $w_i r_i + y_0 = r_{min}^i$ for $r_i = r_{min}^i$ and $w_i r_i + y_0 = \epsilon^2$ for $r_i = r_{max}^i$ holds:
	Subtracting the first from the second equation yields:
	\begin{equation}
		\begin{aligned}	
			&w_i (r_{max}^i - r_{min}^i) = \epsilon^2 - r_{min}^i\\
			\iff &w_i  = \frac{\epsilon^2 - r_{min}^i}{r_{max}^i - r_{min}^i}
		\end{aligned}
	\end{equation}
	and solving for $y_0$ yields:
	\begin{equation}
		\begin{aligned}	
			w_i r_{min}^i + y_0 &= r_{min}^i\\
			\iff y_0 &= r_{min}^i - w_i r_{min}^i\\
			\iff y_0 &= (1- w_i) r_{min}^i
		\end{aligned}
	\end{equation}
	
	Therefore, $w_k r_k + (1 - w_k) r_{min}^k \leq \epsilon^2$ holds for all $k$ for which $o_k = 0$ and therefore these terms are not truncated and thus convex. The sum
	
	\begin{equation}
		\begin{aligned}	
			\sum_{i=1}^{N} (1 - o_i) \left(w_i r_i  + (1 - w_i) r_{min}^i \right)
		\end{aligned}
	\end{equation}
	
	is therefore convex since sums of convex functions are convex \cite[p.36]{Boyd_Vandenberghe_2004}.
	
	Since $\sum_{i=1}^{N} o_i\epsilon^2$ is constant, the overall term \ref{eq:tls-wls-relaxation} also remains convex. Since $r_k \in [r^k_{min}, r^k_{max}]$ and $w_k \leq 1$, $w_k r_k  + (1 - w_k) r_{min}^k \leq r_k$ for all $k$ for which $o_k = 0$, Eq. \ref{eq:tls-wls-relaxation} is below the objective \ref{eq:tls-t-opt-branch2}. Therefore, \ref{eq:tls-wls-relaxation} is a convex relaxation of \ref{eq:tls-t-opt-branch2} over $\mathcal{X}$.
	
\end{proof}	

\section{Interval analysis of residuals under translation}
\label{proof:tls-relax-minmax-r}

Proof of Theorem \ref{thm:interval-analysis-tranlation}.
\begin{proof}
	With $\Delta \vt := \vc - \vt$:
	\begin{equation}
		\begin{aligned}
			  &\min_{\vt \in \MB^t} \normsq{ \vp - \vq + \vt } \\
			= &\min_{\norm{\Delta\vt} \in [0, n_r]} \normsq{ \vp - \vq + \vc - \Delta\vt } \\
			= &\min_{\norm{\Delta\vt} \in [0, n_r]} \normsq{ \vd - \Delta\vt } \\
			= &\min_{\norm{\Delta\vt} \in [0, n_r]} \normsq{\vd} - 2\vd \Delta\vt  + \normsq{\Delta\vt} \\
			= &\max(0, \norm{ \vd } - n_r)^2\\
		\end{aligned}
	\end{equation}
	and similarly:
	\begin{equation}
		\begin{aligned}
			&\max_{\vt \in \MB^t} \normsq{ \vp - \vq + \vt } \\
			= &\max_{\norm{\Delta\vt} \in [0, n_r]} \normsq{\vd} - 2\vd \Delta\vt  + \normsq{\Delta\vt} \\
			= &(\norm{ \vd } + n_r)^2
		\end{aligned}
	\end{equation}
	
\end{proof}

\section{Interval analysis of residuals under rotation}
\label{proof:wls-interval-analysis-rotation}

Proof of Theorem \ref{thm:interval-analysis-rotation}.
\begin{proof}
	With $\MR := \MR_c \Delta\MR$, due to a rotation being an isometry:
	\begin{equation}
		\begin{aligned}
			 &\min_{\MR \in \MB^R} \normsq{ \MR \vp_i - \vq_i } \\
			= &\min_{\MR \in \MB^R} \normsq{ \MR_c^\transposed \left( \MR \vp - \vq \right) } \\
			= &\min_{\MR \in \MB^R} \normsq{ \MR_c^\transposed \left( \MR \vp - \vq \right) } \\
			= &\min_{d_{\angle}(\MI, \Delta \MR) \in [0, n_r]} \normsq{ \Delta\MR \vp - \MR_c^\transposed  \vq } \\
			= &\min_{d_{\angle}(\MI, \Delta \MR) \in [0, n_r]} \left( \Delta\MR \vp - \MR_c^\transposed  \vq \right)^\transposed
				\left( \Delta\MR \vp - \MR_c^\transposed  \vq \right)\\ 
			= &\min_{d_{\angle}(\MI, \Delta \MR) \in [0, n_r]} \left(\Delta\MR\vp \right)^\transposed \left(\Delta\MR\vp \right) - 2 (\Delta\MR\vp )^\transposed \MR_c^\transposed\vq + (\MR_c^\transposed\vq )^\transposed (\MR_c^\transposed\vq )\\
			= &\min_{d_{\angle}(\MI, \Delta \MR) \in [0, n_r]} \vp^\transposed \Delta\MR^\transposed \Delta\MR\vp - 2 (\Delta\MR\vp )^\transposed \MR_c^\transposed\vq + \vq^\transposed  \MR_c \MR_c^\transposed\vq\\
			= &\min_{d_{\angle}(\MI, \Delta \MR) \in [0, n_r]} \normsq{\vp} - 2 (\Delta\MR\vp )^\transposed \MR_c^\transposed\vq + \normsq{\vq}\\
		\end{aligned}
	\end{equation}
	Since $\Delta \MR$ is defined as rotating a point by at most $n_r$ (definition of geodesic distance $d_{\angle}(\MR, \MR_c) \leq n_r \iff d_{\angle}(\MI, \Delta \MR \leq n_r$) the claim follows.
	The maximum case follows similarly.
\end{proof}

\section{Least-squares rotation estimation with equality-constrained angle}
\label{proof:wls-relax-eq-constrained-rotation}

Proof of Theorem. \ref{thm:constrained-wahba-davenport}.
\begin{proof}
	The proof is based on the Davenports quaternion method for solving the Wahba problem \cite{davenportsQ, 8594296}.
	We define 
	\begin{equation}
		\label{eq:ev-constr-data-mat-a}
		\begin{aligned}
			\MA &= \sin^2 \left(\frac{n_r}{2} \right) \MC, \quad \vg =  \sin \left(\frac{n_r}{2} \right) \cos \left(\frac{n_r}{2} \right) \vzz,\\
			\MC &= \MB + \MB^\transposed,\\
			\MB &= \sum_{i=1}^N w_i \vp_i\vq_i^\transposed \in \Rone^{3 \times 3}, \quad
			\vzz = \sum_{i=1}^N w_i [\vp_i]_{\times}\vq_i \in \Rone^{3}
		\end{aligned}
	\end{equation}
	
	The problem \ref{eq:tls-rwls-a} without the equality constraint, i.e. the Wahba problem, can be stated equivalently using quaternions as \cite{davenportsQ, 8594296}:
	\begin{equation}
		\label{eq:davenports-q-problem}
		\begin{aligned}
			\max_{\vq \in \Rone^4} &\quad \vq^\transposed \MK \vq\\
			\text{subject to} &\quad  ||\vq|| = 1
		\end{aligned}
	\end{equation}
	where $\vq$ is the quaternion representing the rotation and $\MK \in \Rone^{4\times4}$ is a data matrix defined as:

	\begin{equation}
		\label{eq:k-mat}
		\begin{aligned}
			\MK = \left(
			\begin{array}{c|c}
				\begin{array}{ccc}
					& & \\
					& \MC & \\
					& & 
				\end{array}
				&
				\begin{array}{c}
					\vzz \\
				\end{array}
				\\
				\hline
				\begin{array}{ccc}
					& \vzz^\transposed &
				\end{array}
				&
				\trace(\MC)
			\end{array}
			\right) \in \Rone^{4 \times 4}
		\end{aligned}
	\end{equation}
	
	We will separate the optimization over the rotation axis and angle using the following definition:
	
	\begin{equation}
		\label{eq:q-axis-angle}
		\begin{aligned}
			\vq = \left( \sin \left( \frac{\theta}{2} \right) \vn, \cos \left(\frac{\theta}{2} \right) \right)^\transposed
		\end{aligned}
	\end{equation}
	where $\vn \in \Rthree$ is a unit-vector that is the axis of rotation and $\theta$ is the angle of rotation.
	We rewrite the objective function $\vq^\transposed \MK \vq$  of \ref{eq:davenports-q-problem} using this definition: 
	
	\begin{equation}
		\label{eq:dv-q-rew}
		\begin{aligned}
			&\vq^\transposed \MK \vq\\
			&= \sin^2 \left(\frac{\theta}{2} \right) \vn^\transposed \MC \vn + 2 \sin \left(\frac{\theta}{2} \right) \cos \left(\frac{\theta}{2} \right) \vzz^\transposed \vn \\
			&+ \cos^2 \left(\frac{\theta}{2} \right) \trace(\MC)
		\end{aligned}
	\end{equation}
	
	Now the axis $\vn$ and the angle $\theta$ occur separately. 
	By the equality constraint, $\theta$ is already given as $\theta = n_r$. Therefore, the term $ \cos^2 \left(\frac{\theta}{2} \right) \trace(\MC)$ in Eq. \ref{eq:dv-q-rew} is constant and does not affect the minimizer, we will therefore omit it. 
	 By defining the data-matrices $\MA$ and $\vg$ as in Eq. \ref{eq:ev-constr-data-mat-a}, we see that the objective function of problem \ref{eq:rot-est-constrained-qcqp} is equivalent to the one of problem \ref{eq:davenports-q-problem} and therefore equivalent to \ref{eq:tls-rwls-a}, concluding the proof.
\end{proof}

\section{Least-squares rotation estimation with equality-constrained axis}
\label{proof:wls-relax-eq-constrained-rotation-axis}

To solve the Wahba-problem with equality-constrained axis, we need again the separation into axis and angle. We do this again using the quaternion formulation.

\begin{equation}
	\begin{aligned}
		&\mathbf{q}^\transposed \mathbf{K} \mathbf{q}\\
		&=\mathbf{q}^\transposed \left( \mathbf{K}_{1} + \mathbf{K}_{2} \right) \mathbf{q}\\
		&= \mathbf{q}^\transposed \mathbf{K}_{1} \mathbf{q} + \mathbf{q}^\transposed \mathbf{K}_{2} \mathbf{q}\\
		&= \sin^2 \left(\frac{\theta}{2} \right) \mathbf{n}^\transposed \mathbf{C} \mathbf{n} + 2 \sin \left(\frac{\theta}{2} \right) \cos \left(\frac{\theta}{2} \right) \mathbf{z}^\transposed \mathbf{n} + \cos^2 \left(\frac{\theta}{2} \right) \trace(\mathbf{C})
	\end{aligned}
\end{equation}

We need to optimize over rotation angle $\theta$ where the axis $\vn$ is fixed. We rewrite the objective in quadratic form and put all constants in a matrix:

\begin{equation}
	\begin{aligned}
		&\sin^2 \left(\frac{\theta}{2} \right) \mathbf{n}^\transposed \mathbf{C} \mathbf{n} + 2 \sin \left(\frac{\theta}{2} \right) \cos \left(\frac{\theta}{2} \right) \mathbf{z}^\transposed \mathbf{n} + \cos^2 \left(\frac{\theta}{2} \right) \trace(\mathbf{C})\\
		&= \begin{pmatrix}
			\sin(\alpha)& \cos(\alpha)
		\end{pmatrix} 
		\begin{pmatrix}
			c_1 & c_2\\
			c_2 & c_3
		\end{pmatrix} 
			\begin{pmatrix}
				\sin(\alpha)\\
			 \cos(\alpha)
			\end{pmatrix}\\
		&= \vv^\transposed \MP \vv
	\end{aligned}
\end{equation}

With $\alpha = \frac{\theta}{2}$ and the constants

\begin{equation}	
	\begin{aligned}
		c_1 &= \mathbf{n}^\transposed \mathbf{C} \mathbf{n}\\
		c_2 &= \mathbf{z}^\transposed \mathbf{n}\\
		c_3 &= \trace(\mathbf{C})\\
	\end{aligned}
\end{equation}

Now, by observing that $\vv \in \{ [\sin(\alpha), \cos(\alpha) ] : \alpha \in [-\pi, \pi]\} = \{ \vv \in \Rtwo : \norm{\vv} = 1\}$, we can state the optmization problem over the angle equivalently as a problem over an unit-norm vector $\vv$:

\begin{equation}
	\begin{aligned}
		\max_{\vv \in \Rtwo} &\quad \vv^\transposed \mathbf{P} \vv\\
		\text{subject to} &\quad \norm{\vv} = 1
	\end{aligned}
\end{equation}

This problem is solved by the major eigenvector of $\MP$. 

\section{Minimizing the WLS relaxation}
\label{proof:minimizing-wls-relaxation}


\begin{theorem}

	Minimizing the WLS relaxation of problem (\ref{eq:pcr-tls}) is the following least-squares problem with two ball constraints:
	\begin{equation}	
		\begin{aligned}
			\label{eq:wls-pose}
			(\MR^*, \vt^*) =  \argmin_{(\MR, \vt) \in B_r \times B_t}  &\sum_{i=1}^{N} w_i \normsq{\MR \vp_i  - \vq_i  + \vt}\\
		\end{aligned}
	\end{equation}
	With the rotation ball $B_r$ and the translation ball $B_t$ defined as:
	\begin{equation}	
		\begin{aligned}
			B_r &= \{ \MR \in \SOtwo : d_{\angle}(\MR_c, \MR) \leq n_r \} \\
			B_t &= \{ \vt \in \Rthree : \norm{\vc - \vt}  \leq n_t \}	
		\end{aligned}	
	\end{equation}
	It is solved by first solving for rotation
	\begin{equation}	
		\begin{aligned}
			\label{eq:wls-pose-simp}
			\MR^* =  \argmin_{\MR \in \hat{B_r} \cap B_r}  \sum_{i=1}^{N} w_i \normsq{\MR \vxx_i  - \vy_i} 		
		\end{aligned}
	\end{equation}
	where 	
	\begin{equation}	
		\begin{aligned}
			\hat{B_r} = \{ \MR \in \SOtwo : \norm{\bar{\vq} -  \MR \bar{\vp} - \vc} \leq n_t \}
		\end{aligned}
	\end{equation}
	and $\vxx_i = \vp_i - \bar{\vp}$ and $\vy_i = \vq_i - \bar{\vq}$, $\bar{\vp}$ and $\bar{\vq}$ are the weighted centroids of both point clouds:
	
	\begin{equation}	
		\begin{aligned}
			\bar{\vp} = \sum_{i=1}^N w_i \vp_i, \quad \bar{\vq} = \sum_{i=1}^N w_i \vq_i
		\end{aligned}
	\end{equation}
	The optimal translation $\vt^*$ is then given by
	\begin{equation}	
		\begin{aligned}
			\vt^* = \bar{\vq} -  \MR^* \bar{\vp}
		\end{aligned}
	\end{equation}
\end{theorem}

\begin{proof}
	In the following, we assume without loss of generality that the axis of rotation is $\vn^* = [0, 0, 1]^\transposed$. This assumption is possible because we can simply change  initially the coordinate system of the points clouds using any rotation $\MR_n$ so that $\MR_n \vn^* = [0, 0, 1]^\transposed$ holds. In this coordinate system, the set of fixed axis rotations $\mathcal{R}$ is simply the set of 2D rotations. 
	
	The method is a modification of the the well-known Kabsch-Umeyama algorithm \cite{Kabsch-1976-Point-set-alignment, Least-squares-estimation-point-sets-Umeyama-1991, sorkine2017least}.
	We first recall the following necessary stationary condition:
	\begin{equation}	
		\begin{aligned}
			\label{eq:ls-pcr-stationary-condition-translation}
			\vt^* = \bar{\vq} -  \MR^* \bar{\vp}
		\end{aligned}
	\end{equation}
	
	Since this stationary condition as well as primal feasability are both necessary for global optimality, so is the following condition necessary:
	
	\begin{equation}	
		\begin{aligned}
			\label{eq:stationary-cond-rot-ball}
			&\norm{\vc - \vt^*}  \leq n_t \\
			\implies &\norm{\vt^* - \vc} = \norm{\bar{\vq} -  \MR^* \bar{\vp} - \vc} \leq n_t
		\end{aligned}
	\end{equation}
	
	The set of 2D rotations that satisfies this condition  (\ref{eq:stationary-cond-rot-ball})
	
	\begin{equation}	
		\begin{aligned}
			\hat{B_r} = \{ \MR \in \SOtwo : \norm{\bar{\vq} -  \MR \bar{\vp} - \vc} \leq n_t \}
		\end{aligned}
	\end{equation}
	
	is a rotation ball (in this case an interval on the unit circle).
	The global minimizer must therefore satisfy the constraint $\hat{B_r} \cap B_r$. 
	By changing the rotation ball constraint from the original ball $B_r$ to the intersection $\hat{B_r} \cap B_r$, we can
	therefore satisfy the translation ball constraint as well. This effectively eliminates the translation ball constraint. 
	
	The remaining part is based on the same arguments as for the original Kabsch-Umeyama algorithm. 

\end{proof}
	Problem (\ref{eq:wls-pose-simp}) is a ball-constrained rotation problem, that is problem (\ref{eq:tls-rwls}).
	
\section{Contractor}
\label{proof:contractor}
Proof of Theorem \ref{thm:so2-contractor}.

\begin{proof}
	It suffices to prove each term separately, i.e. $B_t \times \SOtwo \cap \mathcal{V}_i = B_t \times a^i \cap \mathcal{V}_i$ for all i $\in \RangeOneToN$.
	Note that the set $\mathcal{V}_i$ is the set of solutions of an (polynomial) inequality system.
	The idea of the proof is to show that this inequality only has solutions for the rotation ball $a_i$ given the translation ball $B_t$.
	For this, we first eliminate the variable $\vt$ by inserting $\vt = \vt_c + \Delta \vt$: 
	\begin{equation}
		\begin{aligned}
			\label{eq:v-set-over-r-given-t-ball}
	 	&\normsq{\MR \vp_i  - \vq_i  + \vt_c + \Delta \vt} \leq \epssq \land \norm{\Delta \vt} \leq t_r\\
 	\iff &\normsq{\MR \vp_i  - \vq_i  + \vt_c} \leq (\epsilon + t_r)^2 
		\end{aligned}
	\end{equation}
	We show the claim by simply showing that this inequality has the solution set $a_i$ over $\MR \in \SOtwo$.
	We first rearrange with $\vb_i = \vq_i  - \vt_c$:
	
	\begin{equation}
		\begin{aligned}
			&\normsq{\MR \vp_i  - \vq_i  + \vt_c} \leq (\epsilon + t_r)^2 \\
		\iff &\normsq{\MR \vp_i  - \left( \vq_i  - \vt_c \right)} \leq (\epsilon + t_r)^2 \\
		\iff &\normsq{\vp_i} + \normsq{\vb_i}  - 2 \left( \MR  \vp_i \right)^\transposed \vb_i \leq (\epsilon + t_r)^2 \\
		\iff &\underbrace{ \frac{\normsq{\vp_i} + \normsq{\vb_i} - (\epsilon + t_r)^2}{2 \norm{\vp_i} \norm{\vb_i}} }_{h_i} \leq \cos(\angle(\MR\vp_i, \vb_i))\\
		\iff &h_i \leq \cos(\angle(\MR\vp_i, \vb_i))
		\end{aligned}
	\end{equation}
	Now since $\cos(\cdot) \in [-1, 1]$, we can write equivalently: 
	\begin{equation}
		\begin{aligned}
			\label{eq:v-set-over-r-given-t-ball-simp}
		 			&h_i \leq \cos(\angle(\MR\vp_i, \vb_i))\\
			\iff &\min(1, \max(-1, h_i)) \leq  \cos(\angle(\MR\vp_i, \vb_i)) \\
			\iff &\angle(\MR\vp_i, \vb_i) \leq \underbrace{ \arccos(\min(1, \max(-1, h_i))) }_{r_r^i} \\
			\iff &\angle(\MR\vp_i, \vb_i) \leq r_r^i
		\end{aligned}
	\end{equation}
	Therefore, the solution set of Eq. \ref{eq:v-set-over-r-given-t-ball} is the same as of Eq. \ref{eq:v-set-over-r-given-t-ball-simp}.
	We define the center of a rotation ball $R_c^i = \angle(\vp_i, \vb_i)$. 
	We define the ball of rotations $a^i = \{ \MR \in \SOtwo : d_{\angle}(\MR, R_c^i) \leq r_r^i\}$. We see that $a^i$ is the solution set of Eq. \ref{eq:v-set-over-r-given-t-ball-simp}, concluding the proof.
	
\end{proof}

\section{Computing the root box}
\label{proof:root-box}

	The globally optimal translation of the TLS problem \ref{eq:pcr-tls} is contained in a box with the corner points 
	$\vu = (u_1,u_2,u_3)^\top,\quad \vv=(v_1,v_2,v_3)^\top$, that are defined given the two point clouds 
	$\MP = \bigl[\vp_1,\dots,\vp_N\bigr]\in \RthreeByN$ and $\MQ = \bigl[\vq_1,\dots,\vq_N\bigr]\in \RthreeByN$ as:
	
	\begin{equation}
		\begin{aligned}
			u_i &= \min_{1\le j\le N}\bigl(q_{ij}-\|\vp_j\|\bigr)\;-\; \epsilon \\	
			v_i &= \max_{1\le j\le N}\bigl(q_{ij}+\|\vp_j\|\bigr)\;+\; \epsilon
		\end{aligned}
	\end{equation}

\end{appendix}